\documentclass{article} 

\usepackage{iclr2027_conference,times}

\usepackage{url}

\usepackage{amsmath}
\usepackage{amssymb}
\usepackage{amsthm}
\usepackage{booktabs}
\usepackage{multirow}
\usepackage[table]{xcolor}
\usepackage{graphicx}
\usepackage{subcaption}

\usepackage{listings}
\usepackage{inconsolata}   

\usepackage{fancyhdr}
\fancypagestyle{firstpage}{%
  \fancyhf{}%
  \renewcommand{\headrulewidth}{0pt}%
  \renewcommand{\footrulewidth}{0pt}%
  \lfoot{\footnotesize Preprint.}%
  \cfoot{\thepage}%
}

\usepackage[most]{tcolorbox}

\definecolor{mydarkblue}{HTML}{4472C4}
\definecolor{mydarkgreen}{HTML}{228B22}
\definecolor{mycitecolor}{HTML}{4472C4}
\usepackage[colorlinks,
  linkcolor=mydarkblue,
  citecolor=mydarkblue,
  urlcolor=mydarkblue]{hyperref}

\usepackage{amsthm}

\newtheorem{proposition}{Proposition}

\usepackage{wrapfig}

\newcommand{\DKL}{D_{\mathrm{KL}}}
\newcommand{\E}{\mathbb{E}}

\title{Tail-Aware Top-$k$ On-Policy Distillation}

\author{Huipeng Huang\textsuperscript{1},\enspace
  Hongxin Wei\textsuperscript{1}\thanks{Corresponding author (\texttt{weihx@sustech.edu.cn})}\\
  \textsuperscript{1}Department of Statistics and Data Science, Southern University of Science and Technology 
}

\iclrfinalcopy 

\begin{document}

\maketitle
\thispagestyle{firstpage}

\begin{abstract}
On-policy distillation (OPD) has emerged as an effective paradigm for transferring knowledge between language models, where a student is trained to align its next-token distribution with the teacher's along its own trajectories.
To provide dense supervision at tractable cost, many works minimize the reverse Kullback-Leibler (KL) divergence between the student and teacher's normalized distributions over the teacher's top-$k$ tokens.
However, this normalized objective discards the information about tail probability: the total probability outside the teacher's top-$k$ tokens.
As a result, the optimization can steadily increase the student's tail probability and entropy, empirically degrading downstream accuracy.
To address this issue, we propose Tail-Aware Top-$k$ OPD (\textbf{TA-OPD}), a novel distillation method that restores the missing tail probability signal.
In particular, TA-OPD minimizes the reverse KL divergence over the top-$k$ tokens plus a tail token that carries the tail probability.
In effect, TA-OPD better aligns the student's next-token distribution with the teacher's, preventing the increase in tail probability and entropy caused by top-$k$ normalization.
Extensive experiments demonstrate the superiority of TA-OPD, improving Avg@8 by up to 8.05 points on common benchmarks.
Our code is available at \href{https://github.com/HuipengHuang/TA-OPD}{https://github.com/HuipengHuang/TA-OPD}.
\end{abstract}

\section{Introduction}
Knowledge distillation~\citep{hinton2015distilling} is a promising approach for transferring the capabilities of large language models (LLMs) to smaller models.
However, standard distillation is performed off-policy: the student is trained on teacher-generated sequences~\citep{kim2016sequence, gu2024minillm}, yet at inference time it must condition on its own generations.
This discrepancy between training and inference yields errors that accumulate quickly along the generated sequence~\citep{bengio2015scheduled, arora2022exposure} and hinders effective learning~\citep{xu2025speculative}.
This motivates on-policy distillation (OPD)~\citep{agarwal2024policy, lu2025onpolicydistillation, xu2025kdrl}, where the student is trained to align its next-token distribution with the teacher's along its own trajectories.

A popular idea, normalized top-$k$ OPD~\citep{fu2026revisiting, xing2026trust}, minimizes the reverse Kullback-Leibler (KL) divergence between the student and teacher's normalized distributions over the teacher’s top-$k$ tokens.
However, it discards the information about tail probability: the total probability mass outside the teacher's top-$k$ tokens.
Theoretically, the optimization increases the student's tail probability whenever the student's distribution is not well aligned to the teacher's, which commonly happens in distillation.
Empirically, as the student's tail probability steadily increases, it more frequently samples tokens outside the teacher's top-$k$ tokens, drifting toward prefixes where the teacher's supervision is unreliable.
Consequently, the student fails to imitate the teacher's next-token distribution, degrading downstream performance.


To address this, we propose Tail-Aware Top-$k$ OPD (\textbf{TA-OPD}), a novel distillation method that restores the missing tail probability signal.
In particular, TA-OPD minimizes the reverse KL divergence over the top-$k$ tokens plus a tail token that carries the tail probability.
In effect, TA-OPD aligns the student's tail probability with the teacher's, preventing the tail probability and entropy increase caused by top-$k$ normalization.
Theoretically, we show that TA-OPD's objective is a tight lower bound of the full-vocabulary reverse KL divergence.
In addition, when the sampled token's probability is available, we derive a sampled variant of TA-OPD that debiases the TA-OPD objective with the sampled token and yields an unbiased estimate of the full-vocabulary reverse KL divergence.


Extensive experiments on mathematical benchmarks demonstrate the superiority of TA-OPD over baselines.
Notably, when distilling OpenThinker3-7B~\citep{guha2025openthoughtsdatarecipesreasoning} into Qwen2.5-7B-Instruct~\citep{qwen2025qwen25technicalreport}, normalized top-$k$ OPD's training collapses: its training entropy increases to around 6, and its Avg@8 accuracy on MATH500~\citep{hendrycks2021math} drops to 68.78\%.
In contrast, TA-OPD stabilizes distillation, keeps the training entropy below 1.5, and improves Avg@8 accuracy on MATH500 to 77.88\%.
Through additional analyses, we show that TA-OPD is most effective when the student-teacher capability gap is large, and TA-OPD can be applied with a small $k$.

We summarize our contributions as follows:

\begin{itemize}

\item We show that normalized top-$k$ OPD discards the tail probability.
As a result, minimizing the normalized objective can steadily increase the student's tail probability and entropy, empirically degrading downstream accuracy.

\item We propose TA-OPD, a novel distillation method that restores the tail probability signal.
We show that TA-OPD addresses the increase in tail probability and entropy caused by top-$k$ normalization.
When the sampled token's probability is available, we derive a sampled variant of TA-OPD, an unbiased estimate of the full-vocabulary reverse KL divergence.

\item We conduct extensive experiments to show that TA-OPD achieves superior performance
compared with normalized top-$k$ OPD, improving Avg@8 accuracy by up to 8.05 points on common benchmarks.

\end{itemize}

\section{Preliminaries}
\label{sec:preliminaries}

On-policy distillation (OPD) aims to distill knowledge from a teacher LLM $\pi_{\text{te}}$ to a student LLM $\pi_\theta$ on trajectories sampled from the student.
Given a prompt $x\sim\mathcal{D}_x$, the student samples a response $\hat{y} = (\hat{y}_1, \ldots, \hat{y}_T) \sim \pi_\theta(\cdot \mid x)$.
At each step $t$, on the prefix $\hat{y}_{<t} = (\hat{y}_1, \ldots, \hat{y}_{t-1})$, we define the student's and teacher's next-token distributions over a vocabulary $\mathcal{V}$ as  $p_t = \pi_\theta(\cdot \mid x, \hat{y}_{<t})$ and $q_t = \pi_{\text{te}}(\cdot \mid x, \hat{y}_{<t})$.
OPD minimizes reverse KL divergence over student-sampled trajectories:
\begin{equation}
  \mathcal{L}_{\mathrm{OPD}}(\theta)
  = \mathbb{E}_{x \sim \mathcal{D}_x,\; \hat{y} \sim \pi_\theta(\cdot \mid x)}
    \left[ \sum_{t=1}^{T} D_{\mathrm{KL}}(p_t \| q_t) \right],
  \label{eq:opd_exact_token_decomp}
\end{equation}
where $D_{\mathrm{KL}}(p_t \| q_t) = \sum_{v \in \mathcal{V}} p_t(v) \log \tfrac{p_t(v)}{q_t(v)}$.
Different OPD methods vary in how they approximate $D_{\mathrm{KL}}(p_t \| q_t)$ with their per-token losses $\ell_t$.
Given a student-sampled token $\hat{y}_t \sim p_t$, sampled-token OPD~\citep{lu2025onpolicydistillation, xiao2026mimo} computes
$\ell_t^{\mathrm{sample}} = \log p_t(\hat{y}_t) - \log q_t(\hat{y}_t)$,
which is cheap to compute but discards the dense information over the remaining vocabulary.
Full-vocabulary OPD~\citep{xu2026deepseek} computes $\ell_t^{\mathrm{full}} =D_{\mathrm{KL}}(p_t \| q_t)$ exactly over $\mathcal{V}$.
It provides dense supervision but is prohibitively expensive for LLMs.

Normalized top-$k$ OPD provides an intermediate design between sampled-token and
full-vocabulary OPD by restricting the divergence computation to a subset.
Prior works adopt two choices of the subset: the student's top-$k$ tokens~\citep{li2026rethinking} or the teacher's top-$k$ tokens~\citep{jin2026entropy, fu2026revisiting}.
In this work, we primarily focus on the teacher top-$k$ variant.
Formally, we define the teacher's top-$k$ tokens as $S_t^k= \operatorname{TopK}(q_t, k)$.
The normalized distributions on $S_t^k$ are given by:
\[
\bar{p}_t^{(S_t^k)}(v) = \frac{p_t(v)\,\mathbf{1}[v \in S_t^k]}{\sum_{u \in S_t^k} p_t(u)},
\qquad
\bar{q}_t^{(S_t^k)}(v) = \frac{q_t(v)\,\mathbf{1}[v \in S_t^k]}{\sum_{u \in S_t^k} q_t(u)},
\]
where $\mathbf{1}[\cdot]$ denotes the indicator function.
Distillation is then performed by minimizing the subset reverse KL divergence $\ell_t^{\mathrm{norm}} = D_{\mathrm{KL}}\!\bigl(\bar{p}_t^{(S_t^k)} \,\|\, \bar{q}_t^{(S_t^k)}\bigr)$, yielding the trajectory-level objective:
\begin{equation}
\mathcal{L}_{\mathrm{OPD}}^{\mathrm{norm}}(\theta) = \mathbb{E}_{x \sim \mathcal{D}_x,\; \hat{y} \sim \pi_\theta(\cdot \mid x)} \left[ \sum_{t=1}^{T} D_{\mathrm{KL}}\!\bigl(\bar{p}_t^{(S_t^k)} \,\|\, \bar{q}_t^{(S_t^k)}\bigr) \right]. \label{eq:opd_topk_obj_prelim} 
\end{equation}




Normalized top-$k$ OPD reduces the teacher-query cost while retaining dense, multi-token supervision over $S_t^k$, making it a practical approximation to full-vocabulary OPD.
However, the objective discards the tail probability, i.e., the total probability mass outside the teacher's top-$k$ tokens:
\begin{equation}
p_t^{\mathrm{tail}} = \sum_{v \notin S_t^k} p_t(v),
\qquad
q_t^{\mathrm{tail}} = \sum_{v \notin S_t^k} q_t(v).
\end{equation}
Concerningly, the objective can be minimized even when the student's tail probability is much higher than the teacher's, making the student's next-token distribution diverge substantially from the teacher's.
We proceed by analyzing how the method affects the student's tail probability.

\section{Motivation}
In this section, we investigate how the normalized objective in Eq.~\eqref{eq:opd_topk_obj_prelim} changes the student's tail probability.
We find that minimizing the objective can steadily increase the student's tail probability.

\subsection{Theoretical Analysis on the Tail Probability}
\label{sec:motivation_theory}

Fix a step $t$ and let $z_{t,v}$ denote the
student logit for token $v$ under the prefix, so that
$p_t(v) = \frac{e^{z_{t,v}}}{\sum_{u\in\mathcal{V}} e^{z_{t,u}}}$.
We have the following propositions for the optimization of the normalized objective.

\begin{proposition}
\label{prop:topk_gradient}
For every step $t$, the gradient of $\ell_t^{\mathrm{norm}}$ with respect to
the student logit $z_{t,v}$ is
\[
\frac{\partial \ell_t}{\partial z_{t,v}} =
\begin{cases}
\bar{p}_t^{(S_t^k)}(v)\!\left[\log\dfrac{\bar{p}_t^{(S_t^k)}(v)}{\bar{q}_t^{(S_t^k)}(v)} - \ell_t\right],
& v \in S_t^k,\\[1.0em]
0, & v \notin S_t^k.
\end{cases}
\]
Moreover, the gradients on the logits of the top-$k$ tokens sum to zero:
$\sum_{v\in S_t^k}\frac{\partial \ell_t}{\partial z_{t,v}} = 0$.
\end{proposition}
The proof is provided in  Appendix~\ref{app:topk_gradient}.
The proposition shows that the normalized objective provides no explicit mechanism for decreasing the student's tail probability.
Decreasing the tail probability requires raising the top-$k$ logits relative to the tail logits. 
However, the normalized objective can do neither: the tail logits receive zero gradient, and the gradients on the top-$k$ logits sum to zero.
In the following, we further explore how the zero-sum gradient within the top-$k$ tokens affects the student's tail probability.


\begin{proposition}
\label{prop:tail_prob_increase}
Assume that the student $\pi_\theta$ is a tabular softmax policy, where each token $v$ at each step $t$ is
associated with an independent logit parameter $z_{t,v}=\theta_{t,v}$.
Write $\bar p:=\bar p_t^{(S_t^k)}$ and $\bar q:=\bar q_t^{(S_t^k)}$ for the normalized
distributions on $S_t^k$, and let $\operatorname*{Cov}_{v\sim r}(f,g)$ denote the
covariance of $f(v)$ and $g(v)$ under $v\sim r$.
After one gradient descent step
$\theta \leftarrow \theta - \eta\,\frac{\partial\ell_t^{\mathrm{norm}}}{\partial \theta}$
with learning rate $\eta$, the student's tail probability change $\Delta p_t^{\mathrm{tail}}$ is given by:
\[
\Delta p_t^{\mathrm{tail}}
= \eta\, p_t^{\mathrm{tail}}\bigl(1-p_t^{\mathrm{tail}}\bigr)
\Bigl[
\underbrace{\operatorname*{Cov}_{v\sim\bar p}\bigl(\bar p(v),\log\bar p(v)\bigr)}_{\text{ \normalsize student self-covariance}}
-
\underbrace{\operatorname*{Cov}_{v\sim\bar p}\bigl(\bar p(v),\log\bar q(v)\bigr)}_{\text{\normalsize student-teacher covariance}}
\Bigr] + O(\eta^2).
\]
Consequently, the student's tail probability increases under first-order approximation
if and only if the student self-covariance exceeds the student--teacher covariance.
\end{proposition}

The proof is provided in  Appendix~\ref{app:tail_prob_increase}.
The proposition\footnote{
We use the proposition as a motivation rather than a formal theoretical guarantee.
} shows that minimizing the normalized objective strictly increases the student's tail probability if the student self-covariance exceeds the student--teacher covariance.
The student self-covariance is nonnegative, since $\log \bar p$ is monotone with $\bar p$.
In contrast, the student-teacher covariance is positive only when the student concentrates its probability on the teacher's high-probability tokens within $S_t^k$, and can be negative when the two distributions are poorly aligned.

In practice, when the student-teacher capability gap is large, the student cannot imitate the teacher's next-token distribution well, and its probability mass cannot concentrate on the teacher's high-probability tokens.
As a result, the student--teacher covariance tends to be smaller than the student self-covariance, making the normalized objective exhibit a systematic bias toward increasing the student's tail probability.
This analysis motivates us to empirically examine the existence of the tail probability increase and its influence on the training dynamics of OPD.

\subsection{Empirical Study on the Tail Probability Increase}
\label{subsec:empirical_anlysis}
\paragraph{Setup.}
We conduct experiments with three student-teacher model pairs:
Qwen3-1.7B~\citep{yang2025qwen3} with Qwen3-30B-A3B-Instruct-2507, Qwen2.5-7B-Instruct with OpenThinker3-7B, and
Llama-3.1-8B~\citep{grattafiori2024llama} with DeepSeek-R1-Distill-Llama-8B~\citep{guo2025deepseek}.
All models are trained on DAPO-MATH-17K~\citep{yu2026dapo} with $k=16$.


\begin{figure}
    \centering
    \includegraphics[width=1.0\linewidth]{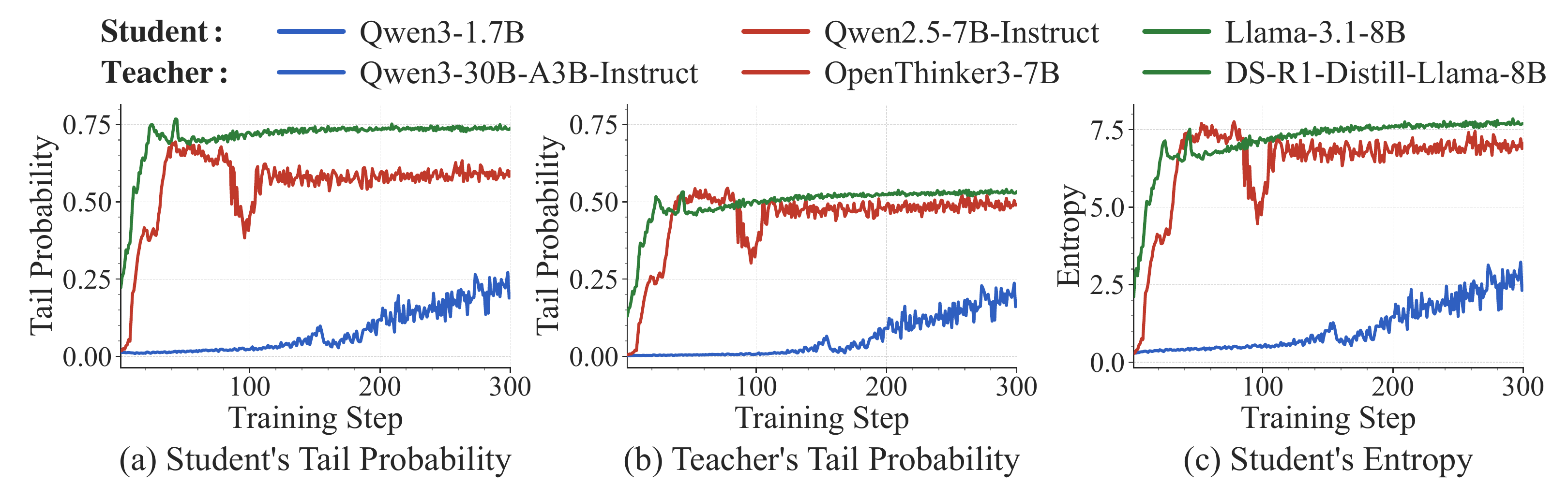}
    \caption{
\textbf{The student's tail probability, the teacher's tail probability, and the student's token-level entropy across training steps.}
Normalized top-$k$ OPD consistently increases the student's tail probability and entropy across all three model pairs.
}
    \label{fig:mass_entropy}
\end{figure}

\paragraph{Normalized top-$k$ OPD steadily increases the student's tail probability.}
Figure~\ref{fig:mass_entropy} visualizes the students' and teachers' tail probabilities on student-generated prefixes, together with the students' entropy across training steps. 
The results show that across three model pairs, the student's tail probability and entropy steadily increase during on-policy training, validating our theoretical analysis.
Furthermore, as the student's tail probability increases, it more frequently samples tokens outside the teacher's top-$k$ tokens.
This drives the student toward reaching prefixes where the teacher is uncertain, as reflected by the increase in the teacher's tail probability.
Prior works show that the teacher's supervision is unreliable on such uncertain prefixes~\citep{fu2026revisiting, xie2026position}, as the teacher itself exhibits significantly reduced accuracy given these prefixes.
Consequently, normalized top-$k$ OPD progressively shifts probability mass away from the teacher's top-$k$ tokens, making the student fail to imitate the teacher's next-token distribution.


\begin{wrapfigure}{r}{0.35\textwidth}
\vspace{-\intextsep}
\centering
\includegraphics[width=\linewidth]{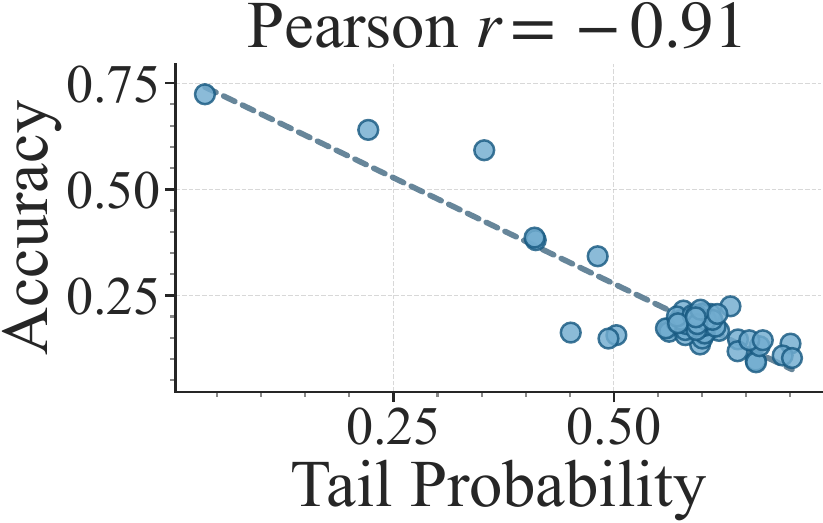}
\caption{\textbf{Validation accuracy versus student's tail probability.}}
\label{fig:tail_vs_acc}
\vspace{-\intextsep}
\end{wrapfigure}
\paragraph{Tail probability increase hurts downstream performance.}
Figure~\ref{fig:tail_vs_acc} plots validation accuracy against the student's tail probability. Each point is a checkpoint from a normalized top-$k$ OPD run with Qwen2.5-7B-Instruct as the student and OpenThinker3-7B as the teacher.
The validation set is MATH-500, and the validation sampling temperature is $1$.
The validation accuracy degrades significantly from above $0.7$ to below $0.2$ as the tail
probability grows from near $0$ to around $0.7$, with a Pearson
correlation coefficient of $-0.91$.

In Appendix~\ref{sec:additional_analysis}, we conduct a detailed analysis on the tail probability increase phenomenon.
We find that it occurs when there is a large capability gap between the student and teacher and when $k$ is not sufficiently large (e.g., $k\leq64$).
Moreover, the phenomenon is more pronounced when the maximum training response length becomes longer.

\begin{figure*}
    \centering
    \includegraphics[width=1.0\linewidth]{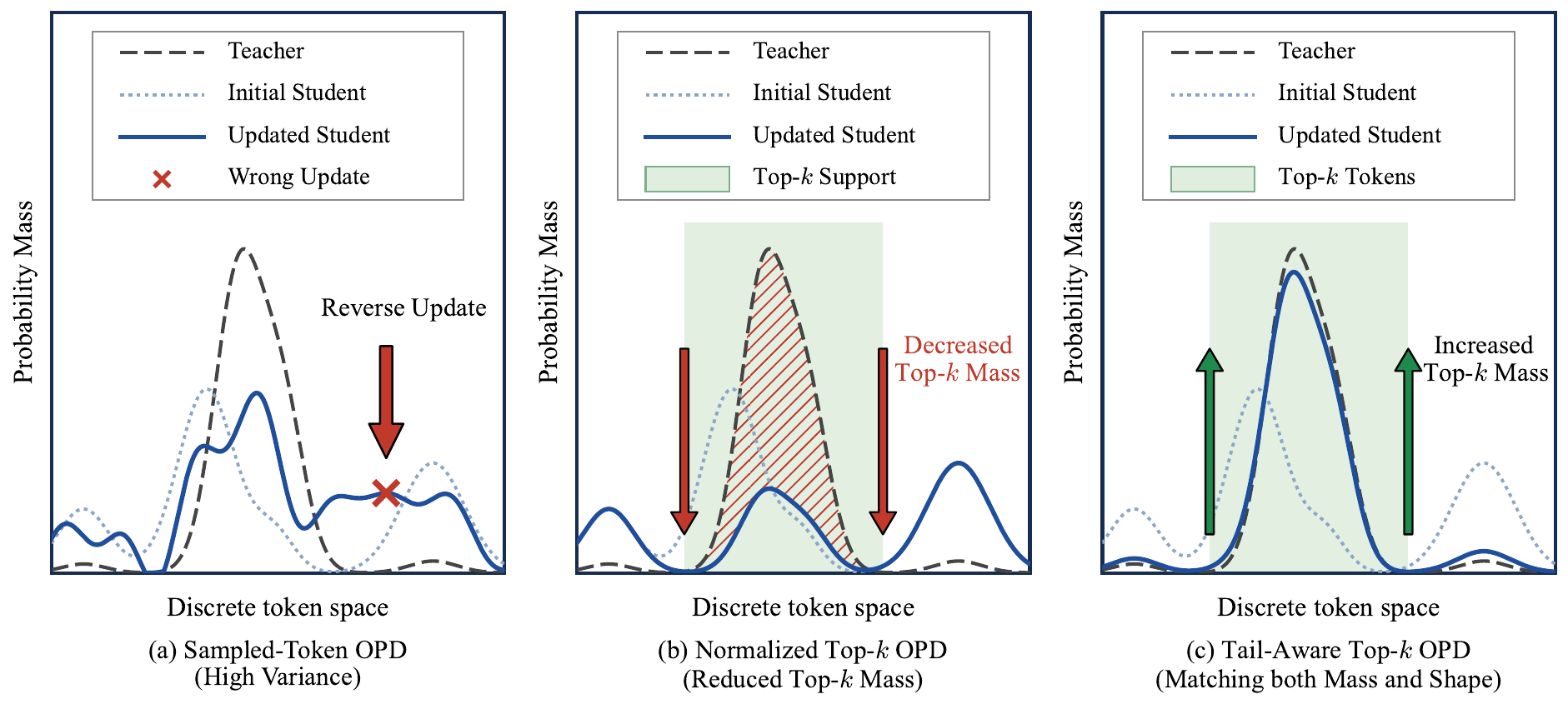}
    \caption{\textbf{Geometric Interpretation of OPD Objectives.}
\textbf{(a) Sparse Update (Sampled-Token OPD):} The supervision is derived from a single sampled token, making the update direction stochastic and prone to updating some token logits in erroneous directions.
\textbf{(b) Increased Tail Probability (Normalized Top-$k$ OPD):} The normalized objective discards the tail probability, leading to increased tail probability.
\textbf{(c) Decreased Tail Probability (Tail-Aware Top-$k$ OPD, Ours):}  TA-OPD restores the missing tail probability signal, matching the student's tail probability to the teacher's.
}

    \label{fig:geometry}
\end{figure*}

\section{Method}
In our previous analysis, we show that the normalized objective discards information about the tail probability, leading to the tail probability increase issue (Figure~\ref{fig:geometry} (b)). To address this problem, our key idea is to restore the missing tail probability signal.


\paragraph{Tail-Aware Top-$k$ OPD.}
We propose Tail-Aware Top-$k$ OPD (TA-OPD), which explicitly aligns the student's tail probability with the teacher's (Figure~\ref{fig:geometry} (c)). 
In particular, we introduce a tail token $v_{\mathrm{tail}}$ that carries the tail probability.
We then minimize the reverse KL divergence over the top-$k$ tokens plus
the tail token.
Formally, we define the augmented token set
$S_t^{+}=S_t^k\cup\{v_{\mathrm{tail}}\}$.
On $S_t^{+}$, each top-$k$ token retains its original probability,
$p_t(v)=\pi_\theta(v\mid x,\hat y_{<t})$, while the tail token carries the
tail probability, $p_t(v_{\mathrm{tail}})=p_t^{\mathrm{tail}}$.
The per-token loss function of TA-OPD is then given by:
\begin{equation}
\ell_t^{\mathrm{TA}}
= \sum_{v\in S_t^{+}} p_t(v)\log\frac{p_t(v)}{q_t(v)}
= \sum_{v\in S_t^k} p_t(v)\log\frac{p_t(v)}{q_t(v)}
  + p_t^{\mathrm{tail}}\log\frac{p_t^{\mathrm{tail}}}{q_t^{\mathrm{tail}}}.
\label{eq:lma}
\end{equation}
By introducing the tail token $v_{\mathrm{tail}}$, the term $p_t^{\mathrm{tail}}\log\frac{p_t^{\mathrm{tail}}}{q_t^{\mathrm{tail}}}$ explicitly compares the student's and teacher's tail probabilities, thereby driving the student's tail probability toward the teacher's.
We next characterize the gradient signal of $\ell_t^{\mathrm{TA}}$.

\begin{proposition}
\label{prop:ta}
For every step $t$, the gradient of
$\ell_t^{\mathrm{TA}}$ with respect to the student logit $z_{t,v}$ is
\begin{equation}
\frac{\partial \ell_t^{\mathrm{TA}}}{\partial z_{t,v}} =
\begin{cases}
p_t(v)\Big(\log\dfrac{p_t(v)}{q_t(v)}-\ell_t^{\mathrm{TA}}\Big), & v\in S_t^k,\\[1.2em]
p_t(v)\Big(\log\dfrac{p_t^{\mathrm{tail}}}{q_t^{\mathrm{tail}}}-\ell_t^{\mathrm{TA}}\Big), & v\notin S_t^k.
\end{cases}
\label{eq:gma}
\end{equation}
\end{proposition}
The proposition reveals two key properties of TA-OPD. 
First, while the normalized objective gives no gradient to logits outside the top-$k$ tokens, TA-OPD updates them through the tail log-ratio $\log\tfrac{p_t^{\mathrm{tail}}}{q_t^{\mathrm{tail}}}$, which compares the student's tail probability with the teacher's.
Second, the gradients on the top-$k$ logits no longer sum to zero.
The top-$k$ gradient sum
$
\sum_{v\in S_t^k}\frac{\partial \ell_t^{\mathrm{TA}}}{\partial z_{t,v}}
= p_t^{\mathrm{tail}}\Bigl(\ell_t^{\mathrm{TA}}-\log\tfrac{p_t^{\mathrm{tail}}}{q_t^{\mathrm{tail}}}\Bigr)$ is negative whenever the tail log-ratio is larger than the average log-ratio over $S_t^k$.
It therefore can explicitly raise the top-$k$ logits relative to the tail whenever the student over-weights the tail relative to the teacher.
In the following, we show that $\ell_t^{\mathrm{TA}}$ is not merely a heuristic modification but a principled approximation of the full-vocabulary objective.

\begin{proposition}[Lower bound of the full-vocabulary reverse KL]
\label{prop:bound}
Let $\tilde p_t(v)=p_t(v)/p_t^{\mathrm{tail}}$ and $\tilde q_t(v)=q_t(v)/q_t^{\mathrm{tail}}$ for $v\notin S_t^k$ be the normalized tail distributions. Then
\begin{equation}
 \ell_t^{\mathrm{TA}}
=  \ell_t^{\mathrm{full}} - p_t^{\mathrm{tail}}\,D_{\mathrm{KL}}(\tilde p_t\|\tilde q_t)\leq \ell_t^{\mathrm{full}}.
\label{eq:gap}
\end{equation}

\end{proposition}
The proofs of the above propositions are presented in Appendix~\ref{app:ta} and \ref{app:bound}.
This proposition shows that the loss functions of TA-OPD and full-vocabulary OPD differ only by a non-negative residual term $p_t^{\mathrm{tail}}\,D_{\mathrm{KL}}(\tilde p_t\|\tilde q_t)$.
In practice, TA-OPD drives both the student's and teacher's tail probabilities close to zero (see Figure~\ref{fig:mass_entropy_comparison} (a) and (b)). 
As a result, the residual term becomes small, making $\ell_t^{\mathrm{TA}}$ a tight lower bound to the ideal full-vocabulary reverse KL objective.

\section{Experiments}
\subsection{Experimental Setup}

\paragraph{Models and Training Dataset.}
For main experiments, we use the three student-teacher model pairs as in section~\ref{subsec:empirical_anlysis}: Qwen3-1.7B with Qwen3-30B-A3B-Instruct-2507, Qwen2.5-7B-Instruct with OpenThinker3-7B, and
Llama-3.1-8B with DeepSeek-R1-Distill-Llama-8B.
We disable thinking modes for Qwen3 models.
The training set is DAPO-MATH-17K.

\paragraph{Evaluation.}
We evaluate on six math reasoning benchmarks---MATH500, Minerva~\citep{lewkowycz2022solving}, OlympiadBench~\citep{he2024olympiadbench}, AMC~\citep{li2024numinamath} and AIME24/25~\citep{li2024numinamath}---and two out-of-distribution benchmarks, ARC-c~\citep{clark2018think} and MMLU-Pro~\citep{wang2024mmlu}.
We use a rollout temperature of 0.7, top-p sampling with p = 0.95, and a maximum response length of 8192 tokens.
For MMLU-Pro, we report Pass@1. For other datasets, we sample 8 responses per question and report the average accuracy (Avg@8).

\paragraph{Compared Methods.}
We compare our method with sampled-token OPD and normalized top-$k$ OPD.
In Appendix~\ref{subsec:ablate_tail}, we conduct an additional comparison with unnormalized top-$k$ OPD.
We provide a detailed introduction to these methods in Appendix~\ref{sec:methods}.

\paragraph{Implementation.}
All the OPD experiments use the same training setup: 300 training steps with a learning rate of $1\times10^{-6}$.
Unless otherwise specified, we set $k=16$ for top-$k$ OPD.
We use a prompt batch size of 72 and sample 4 rollouts per prompt, with a maximum generation length of 7168 tokens.
More details of implementation are provided in Appendix~\ref{sec:implementation_detail}.

\subsection{Results}

\newcolumntype{C}[1]{>{\centering\arraybackslash}p{#1}}

\begin{table}[t]
\centering
\caption{
\textbf{Results on math reasoning and out-of-distribution (OOD) benchmarks.}
We compare TA-OPD with sampled-token OPD (Sampled-token) and
normalized top-$k$ OPD (Norm.\ top-$k$).
Each \textbf{Avg.} column represents the macro-average.
Best results are shown in \textbf{bold}.
}
\label{tab:main}

\resizebox{\textwidth}{!}{
\setlength{\tabcolsep}{2pt}
\renewcommand{\arraystretch}{1.67}
\begin{tabular}{
    l|
    *{7}{C{1.5cm}}|
    C{1.3cm}C{1.81cm}C{1.3cm}
}
\toprule
\multirow{2}{*}{\textbf{Methods}}
& \multicolumn{7}{c|}{\textbf{In-Distribution Performance}}
& \multicolumn{3}{c}{\textbf{OOD Performance}} \\
\cmidrule(lr){2-8}
\cmidrule(lr){9-11}
& \textbf{MATH500}
& \textbf{Minerva}
& \textbf{Olympiad}
& \textbf{AMC}
& \textbf{AIME 24}
& \textbf{AIME 25}
& \textbf{Avg.}
& \textbf{ARC-c}
& \textbf{MMLU-Pro}
& \textbf{Avg.} \\
\midrule

\multicolumn{11}{c}{
    Student: \textbf{\textit{Qwen2.5-7B-Instruct}}
    \quad
    Teacher: \textbf{\textit{OpenThinker3-7B}}
} \\
\midrule

Sampled-token
& 74.60
& 31.25
& 41.22
& \textbf{45.48}
& 14.58
& 16.25
& 37.23
& 74.99
& 47.66
& 61.33 \\

Norm.\ top-$k$
& 68.78
& 25.55
& 36.02
& 41.42
& 13.75
& 15.00
& 33.42
& 23.07
& 37.38
& 30.23 \\

\rowcolor{blue!10}
\textbf{TA-OPD (Ours)}
& \textbf{77.88}
& \textbf{32.58}
& \textbf{42.41}
& 45.03
& \textbf{16.67}
& \textbf{17.50}
& \textbf{38.68}
& \textbf{76.30}
& \textbf{50.35}
& \textbf{63.33} \\
\midrule

\multicolumn{11}{c}{
    Student: \textbf{\textit{Llama-3.1-8B}}
    \quad
    Teacher: \textbf{\textit{DeepSeek-R1-Distill-Llama-8B}}
} \\
\midrule

Sampled-token
& 43.45
& 12.78
& 16.56
& 17.62
& 2.08
& 1.25
& 15.62
& 41.88
& 28.99
& 35.44 \\

Norm.\ top-$k$
& 31.30
& 9.24
& 9.96
& 11.45
& 1.25
& 0.42
& 10.60
& 10.01
& 18.33
& 14.17 \\

\rowcolor{blue!10}
\textbf{TA-OPD (Ours)}
& \textbf{50.83}
& \textbf{14.94}
& \textbf{20.65}
& \textbf{20.48}
& \textbf{2.50}
& \textbf{2.50}
& \textbf{18.65}
& \textbf{65.70}
& \textbf{34.57}
& \textbf{50.14} \\
\midrule

\multicolumn{11}{c}{
    Student: \textbf{\textit{Qwen3-1.7B}}
    \quad
    Teacher: \textbf{\textit{Qwen3-30B-A3B-Instruct-2507}}
} \\
\midrule

Sampled-token
& 82.98
& 49.49
& \textbf{34.93}
& 52.41
& \textbf{26.25}
& 19.58
& 44.27
& \textbf{83.91}
& 50.08
& 67.00 \\

Norm.\ top-$k$
& 81.80
& 47.20
& 33.28
& 50.60
& 23.33
& 17.08
& 42.22
& 82.48
& 49.66
& 66.07 \\

\rowcolor{blue!10}
\textbf{TA-OPD (Ours)}
& \textbf{83.48}
& \textbf{49.82}
& 34.50
& \textbf{52.86}
& \textbf{26.25}
& \textbf{20.83}
& \textbf{44.62}
& 83.59
& \textbf{50.42}
& \textbf{67.01} \\

\bottomrule
\end{tabular}
}
\end{table}

\paragraph{TA-OPD achieves the best performance across student--teacher model pairs.}
Table~\ref{tab:main} compares TA-OPD with sampled-token OPD and normalized top-$k$ OPD.
Evaluated on six math reasoning benchmarks, TA-OPD attains the best average accuracy on every student--teacher pair, outperforming normalized top-$k$ OPD by \textbf{+5.26} points on Qwen2.5-7B-Instruct and \textbf{+8.05} points on Llama-3.1-8B.
Notably, TA-OPD demonstrates a significantly greater advantage on MATH500 with Llama-3.1-8B over normalized top-$k$ OPD, improving Avg@8 from 31.30 to 50.83  (\textbf{+19.53} points).
Regarding out-of-distribution performance, TA-OPD also demonstrates strong performance gain: on ARC-c, it achieves 76.30 with Qwen2.5-7B-Instruct and 65.70 with Llama-3.1-8B, whereas normalized top-$k$ OPD degrades to 23.07 and 10.01, indicating that the tail probability increase not only hurts mathematical reasoning but also degrades the student's general capabilities.
Overall, these results demonstrate that TA-OPD consistently achieves the best in-distribution and out-of-distribution performance across different model pairs.

\begin{figure}
    \centering
    \includegraphics[width=1.0\linewidth]{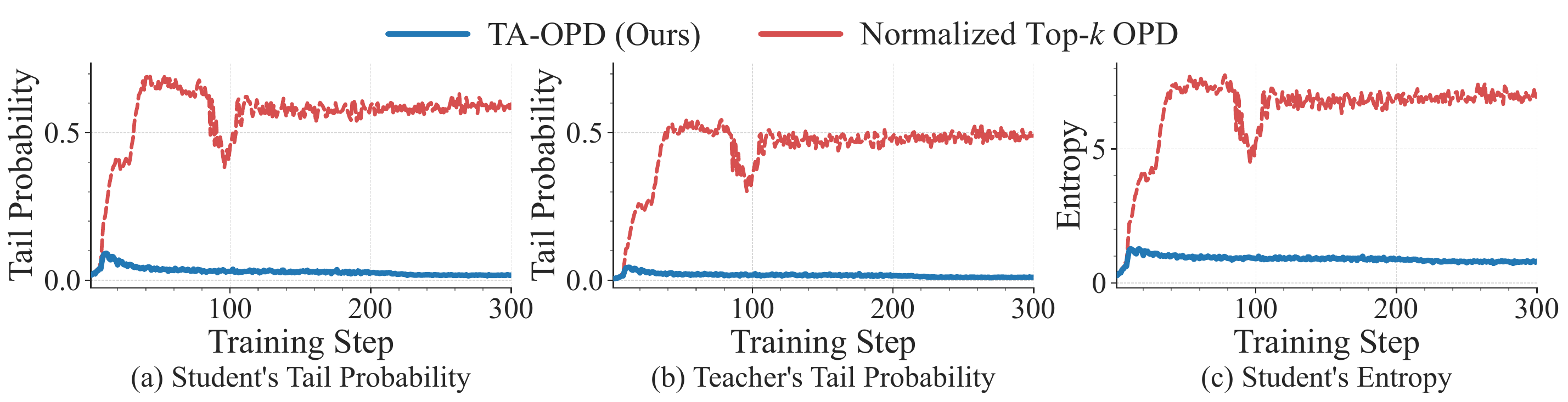}
    \caption{
\textbf{Tail probabilities and token-level entropy of TA-OPD and normalized top-$k$ OPD over training.}
The student and teacher models are Qwen2.5-7B-Instruct and OpenThinker3-7B.
}
    \label{fig:mass_entropy_comparison}
\end{figure}

\paragraph{TA-OPD prevents the increase in tail probability and entropy.}
In Figure~\ref{fig:mass_entropy_comparison}, we compare the training dynamics of TA-OPD and normalized top-$k$ OPD on the Qwen2.5-7B-Instruct and OpenThinker3-7B model pair.
With normalized top-$k$ OPD, the student's tail probability, the teacher's tail probability, and the student's entropy increase to approximately $0.6$, $0.5$, and $6$ over the course of training. 
In contrast, TA-OPD keeps the two tail probabilities close to $0$ and the student's entropy below $1.5$.
The same trends are consistently observed for three alternative model pairs in Figure~\ref{fig:mass_entropy_comparison_multi}.
In short, these results show that TA-OPD addresses the tail probability and entropy increase caused by top-$k$ normalization.


\begin{wrapfigure}[8]{r}{0.32\textwidth}
\vspace{-\baselineskip}
\centering
\includegraphics[width=\linewidth]{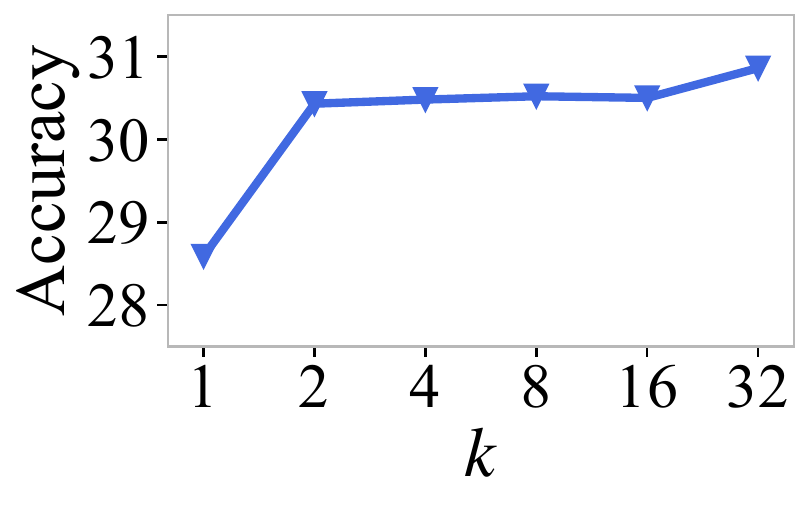}
\vspace{-0.75cm}
\caption{\textbf{Ablation study on $k$.}}
\label{fig:ablate_k_performance}
\end{wrapfigure}
\paragraph{TA-OPD is effective with a small $k$.}
In Figure~\ref{fig:ablate_k_performance}, we ablate how the number of top-$k$ tokens
affects the performance of TA-OPD.
In particular, we fix Qwen3-1.7B-Base as the student and Qwen3-8B as the teacher, and only
vary $k\in\{1,2,4,8,16,32\}$.
For each $k$, we report the average accuracy macro-averaged across the six math reasoning
benchmarks.
The results show that the accuracy is insensitive to $k$ once $k\ge 2$: $k=2$ attains an average accuracy of $30.43$, only $0.43$ points below the best result of $30.86$ at $k=32$.
Overall, TA-OPD can be applied with a small $k$.

\paragraph{Additional results.}
Due to space constraints, we defer additional analyses to the appendix.
Appendix~\ref{subsec:synthetic} conducts a synthetic experiment to provide an intuitive understanding of TA-OPD and normalized top-$k$ OPD's optimization.
Appendix~\ref{subsec:overlap} shows that TA-OPD consistently achieves higher student-teacher top-$k$ overlap ratios compared to normalized top-$k$ OPD.
Appendix~\ref{subsec:overhead} shows that TA-OPD incurs negligible computational overhead over normalized top-$k$ OPD.

\section{Discussion}

\begin{table}[t]
    \centering
    \caption{
    \textbf{Comparison between TA-OPD and sample-corrected TA-OPD on math reasoning and out-of-distribution (OOD) benchmarks.}
    Best results are shown in \textbf{bold}.
    }
    \label{tab:sc_ta}
    \resizebox{\textwidth}{!}{
        \setlength{\tabcolsep}{2pt}
        \renewcommand{\arraystretch}{1.67}
        \begin{tabular}{
            l|
            *{7}{C{1.5cm}}|
            C{1.3cm}C{1.81cm}C{1.3cm}
        }
        \toprule
        \large
        \multirow{2}{*}{\textbf{Methods}}
        & \multicolumn{7}{c|}{\textbf{In-Distribution Performance}}
        & \multicolumn{3}{c}{\textbf{OOD Performance}} \\
        \cmidrule(lr){2-8}
        \cmidrule(lr){9-11}
        & \textbf{MATH500}
        & \textbf{Minerva}
        & \textbf{Olympiad}
        & \textbf{AMC}
        & \textbf{AIME 24}
        & \textbf{AIME 25}
        & \textbf{Avg.}
        & \textbf{ARC-c}
        & \textbf{MMLU-Pro}
        & \textbf{Avg.} \\
        \midrule

        \multicolumn{11}{c}{
            Student: \textbf{\textit{DeepSeek-R1-Distill-Qwen-1.5B}}
            \quad
            Teacher: \textbf{\textit{JustRL-DeepSeek-1.5B}}
        } \\
        \midrule

        TA-OPD
        & 86.13 & 30.97 & 57.04 & 76.20
        & 41.67 & 29.73 & 53.62
        & 35.40 & \textbf{25.98} & \textbf{30.69} \\
        
        SC-TA-OPD
        & \textbf{87.00} & \textbf{32.08} & \textbf{57.65} & \textbf{76.51}
        & \textbf{42.92} & \textbf{30.17} & \textbf{54.39}
        & \textbf{35.55} & 25.76 & 30.66 \\
        \midrule

       \multicolumn{11}{c}{
    Student: \textbf{\textit{Qwen3-1.7B}}
    \quad
    Teacher: \textbf{\textit{Qwen3-30B-A3B-Instruct-2507}}
} \\
\midrule

TA-OPD
& 83.48
& \textbf{49.82}
& 34.50
& 52.86
& 26.25
& \textbf{20.83}
& 44.62
& 83.59
& \textbf{50.42}
& 67.01 \\

SC-TA-OPD
& \textbf{84.00}
& 49.54
& \textbf{34.93}
& \textbf{53.16}
& \textbf{26.67}
& 19.58
& \textbf{44.65}
& \textbf{83.91}
& 50.28
& \textbf{67.10} \\

\bottomrule
        \end{tabular}
    }
\end{table}

\paragraph{Sample-corrected TA-OPD.}
While TA-OPD's loss function is a tight lower bound of the
full-vocabulary reverse KL divergence, it remains a biased estimate of the latter.
By Proposition~\ref{prop:bound}, the bias is
$p_t^{\mathrm{tail}}\,\DKL(\tilde p_t\,\|\,\tilde q_t)$.
When the sampled token's probability is additionally available, we can estimate the bias with the sampled token $\hat y_t$, and add this estimate back to $\ell_t^{\mathrm{TA}}$ to yield an unbiased estimate of the full-vocabulary reverse KL divergence.
Formally, the loss function of sample-corrected TA-OPD is given by:
\begin{equation}
\begin{aligned}
\ell_t^{\mathrm{SC\text{-}TA}}
=&
\sum_{v\in S_t^{+}}
p_t(v)
\operatorname{sg}\!\left(
\log\frac{p_t(v)}{q_t(v)}
\right)
+
\mathbf{1}[\hat y_t\notin S_t^k]\,
\frac{p_t(\hat y_t)}{\operatorname{sg}(p_t(\hat y_t))}
\operatorname{sg}\!\left(
\log\frac{p_t(\hat y_t)/p_t^{\mathrm{tail}}}{q_t(\hat y_t)/q_t^{\mathrm{tail}}}
\right),
\end{aligned}
\label{eq:lmas}
\end{equation}
where $\operatorname{sg}(\cdot)$ denotes the stop-gradient operator.
We next characterize its theoretical property.

\begin{proposition}[Unbiasedness]
\label{prop:unbias}
For every step $t$, taking the randomness over $\hat y_t\sim p_t$, we have
\[
\E_{\hat y_t\sim p_t}\bigl[\ell_t^\mathrm{SC\text{-}TA}\bigr]
=\DKL(p_t\,\|\,q_t),
\quad
\E_{\hat y_t\sim p_t}\bigl[\nabla_\theta\,\ell_t^\mathrm{SC\text{-}TA}\bigr]
=\nabla_\theta\,\DKL(p_t\,\|\,q_t).
\]
\end{proposition}
The proof is provided in Appendix~\ref{app:unbias}.
Proposition~\ref{prop:unbias} shows that $\ell_t^{\mathrm{SC\text{-}TA}}$ is an unbiased estimate of the full-vocabulary reverse KL in both value and gradient, thereby combining the advantages of sampled-token OPD's unbiasedness and top-$k$ OPD's dense supervision.

Table~\ref{tab:sc_ta} compares TA-OPD with its sample-corrected variant on two student--teacher pairs.
Removing the bias term yields only marginal gains: the macro-average on mathematical reasoning improves from $53.62$ to $54.39$ for DeepSeek-R1-Distill-Qwen-1.5B and from $44.62$ to $44.65$ for Qwen3-1.7B, while the OOD averages remain essentially unchanged in both settings ($30.69$ vs.\ $30.66$ and $67.01$ vs.\ $67.10$).
We attribute this to the tightness of TA-OPD: its bias to full-vocabulary OPD $p_t^{\mathrm{tail}}\,\DKL(\tilde p_t\,\|\,\tilde q_t)$ is negligible in practice.
Overall, SC-TA-OPD provides rigorous theoretical guarantees of unbiasedness, while TA-OPD attains comparable performance to SC-TA-OPD.

\begin{figure}
    \centering
    \includegraphics[width=\linewidth]{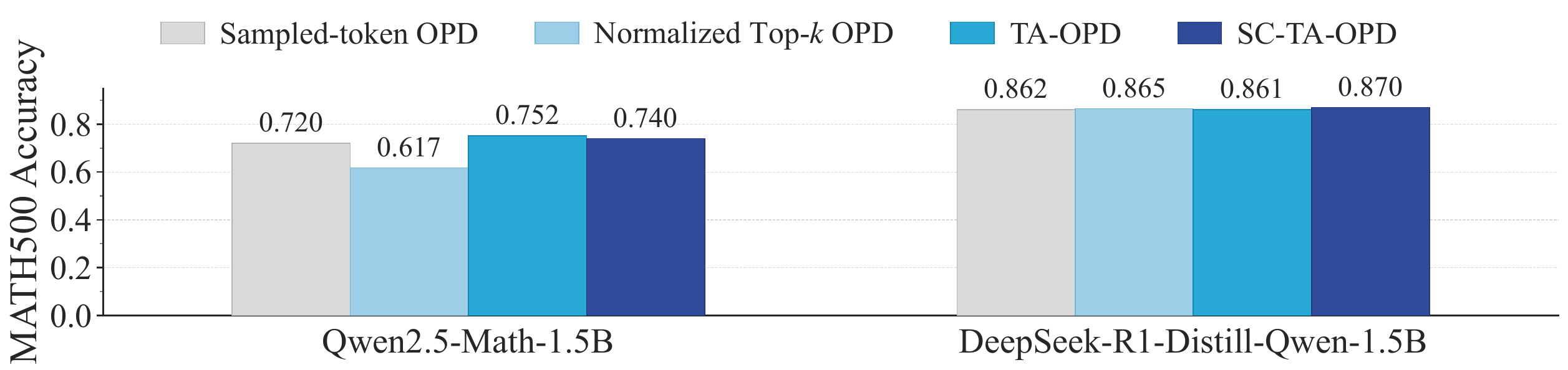}
    \caption{
\textbf{Performance under different student--teacher capability gaps.}
The students are Qwen2.5-Math-1.5B and DeepSeek-R1-Distill-Qwen-1.5B.
The teacher is JustRL-DeepSeek-1.5B. 
}
\label{fig:capability_gap}
\end{figure}

\paragraph{TA-OPD is most effective when the student--teacher capability gap is large.}
We fix JustRL-DeepSeek-1.5B~\citep{he2025justrl} as the teacher and compare two students, Qwen2.5-Math-1.5B and DeepSeek-R1-Distill-Qwen-1.5B.
DeepSeek-R1-Distill-Qwen-1.5B is obtained by applying supervised fine-tuning to Qwen2.5-Math-1.5B, and the teacher is obtained by applying reinforcement learning to DeepSeek-R1-Distill-Qwen-1.5B.
The capability gap between Qwen2.5-Math-1.5B and the teacher is therefore larger than that between DeepSeek-R1-Distill-Qwen-1.5B and the teacher.

Figure~\ref{fig:capability_gap} presents the Avg@8 accuracy of four OPD objectives on MATH500 for the two students.
On Qwen2.5-Math-1.5B, TA-OPD attains an Avg@8 of 75.20 on MATH500, outperforming normalized top-$k$ OPD by \textbf{13.50} points.
On DeepSeek-R1-Distill-Qwen-1.5B, the four objectives achieve similar performance, with sample-corrected TA-OPD attaining the best accuracy of 87.00.
This is because, as shown in Figure~\ref{fig:capability_gap_tail_prob}, the tail
probability increase is more pronounced under the larger capability gap, where TA-OPD
therefore brings a larger improvement.
In summary, TA-OPD brings substantial improvements over baselines when the
student--teacher capability gap is large.

\section{Related Work}
\paragraph{On-policy Distillation.}
On-policy distillation (OPD) is an effective post-training paradigm that has attracted a surge of interest~\citep{song2026survey}.
To provide dense supervision at tractable cost, top-$k$ OPD has become a popular research direction~\citep{zhu2026many, xing2026trust}.
Normalized top-$k$ OPD~\citep{li2026rethinking, fu2026revisiting} directly uses normalized top-$k$ reverse KL as the loss function.
Entropy-aware OPD~\citep{jin2026entropy} improves sampled-token OPD by applying normalized top-$k$ forward KL at high-entropy positions.
vOPD~\citep{oh2026kl} uses normalized top-$k$ reverse KL as a reward baseline to reduce variance of sampled-token OPD.
While these methods use normalized top-$k$ KL for different purposes, they share a common limitation: the normalized objective does not faithfully approximate the full-vocabulary KL. 
We note that ~\citet{dasgupta2026don} study a related tail probability issue, but it considers off-policy full-vocabulary distillation and is not applicable to top-$k$ OPD.
Moreover, the failure mode it studies differs from ours: it studies how full-vocabulary forward KL in the off-policy setting makes the student's tail probability too small, while we study how normalized top-$k$ objectives in the on-policy setting steadily increase it.

\paragraph{KL computation for LLM post-training.}
KL divergence is a crucial component in post-training of LLMs~\citep{vieillard2020leverage, ouyang2022training}.
In OPD, the student is optimized by minimizing the reverse KL divergence to the teacher~\citep{lu2025onpolicydistillation}, while in reinforcement learning (RL), a KL constraint against the base model is commonly imposed to prevent the policy from drifting away~\citep{ziegler2019fine, shah2025comedy}.
For efficiency, most existing methods estimate the divergence using only the sampled tokens~\citep{schulman2015trust, schulman2020approximating}.
These estimators suffer from high variance and ignore the dense information in the logit space. 
To provide dense information at tractable cost, top-k approximations of the KL divergence have become an appealing alternative~\citep{zhang2026ema}.
Existing top-$k$ OPD methods predominantly adopt a normalized top-k formulation~\citep{li2026rethinking, fu2026revisiting}, which, as we reveal in this work, leads to a steady increase in the student's tail probability.
In the context of RL, several top-k KL approximations have also been explored, yet they are devised for purposes different from ours.
For instance, DPPO~\citep{qi2026rethinking} derives a KL approximation similar to TA-OPD's loss objective, but uses it as a quantity to decide whether to clip the policy update on a token, thereby substituting heuristic PPO clipping~\citep{schulman2017proximal} with a principled constraint.
In contrast, we directly use the top-$k$ KL approximation as the optimization objective and establish its theoretical properties.
Besides, EMA-PG~\citep{zhang2026ema} proposes a KL approximation similar to our sample-corrected TA-OPD.
We provide a discussion of sample-corrected TA-OPD and EMA-PG in Appendix~\ref{subsec:ema_pg}.

\section{Conclusion}
In this paper, we introduce Tail-Aware Top-k On-Policy Distillation (TA-OPD), a novel distillation method that restores the missing tail probability signal. 
In particular, TA-OPD minimizes the reverse KL divergence over the teacher's top-$k$ tokens plus a tail token that carries the tail probability.
In effect, TA-OPD explicitly aligns the student's tail probability with the teacher's, addressing the tail probability and entropy increase caused by top-$k$ normalization.
We further derive a sampled variant that yields an unbiased estimate of the full-vocabulary reverse KL when the sampled token's probability is available. 
Extensive experiments show that TA-OPD consistently improves accuracy across benchmarks.
Our method can be easily adopted in practice: it is straightforward to implement with existing OPD frameworks and requires no additional teacher queries beyond the top-$k$ probabilities. 
We hope that our insights inspire future research to further explore loss function designs for OPD.

\paragraph{Limitations.}
The performance gain of TA-OPD over normalized top-$k$ OPD diminishes as $k$ increases or the capability gap between student and teacher decreases.
Additionally, due to limited computational resources, our experiments are restricted to models with up to 8B parameters. 
While these sizes are standard for research-stage OPD studies, we do not presume automatic transfer to 30B+ models.

\bibliography{iclr2027_conference}
\bibliographystyle{iclr2027_conference}

\clearpage

\appendix

\section{Proof}

\subsection{Proof of Proposition~\ref{prop:topk_gradient}}
\label{app:topk_gradient}

\begin{proof}
Fix a step $t$ and a prefix $\hat y_{<t}$, and omit the subscript $t$ throughout: write
$S=S_t^k=\operatorname{TopK}(q_t,k)$ for the teacher's top-$k$ tokens, $z_v=z_{t,v}$ for the
student logit of token $v$, and $\bar p=\bar p_t^{(S_t^k)}$, $\bar q=\bar q_t^{(S_t^k)}$ for
the normalized distributions on $S$, so that
\[
\ell^{\mathrm{norm}}
=\sum_{u\in S}\bar p(u)\log\frac{\bar p(u)}{\bar q(u)} .
\]
Since the token set $S$ is selected from the teacher distribution, both $S$ and $\bar q$ are
constants with respect to the student logits.
We proceed in three steps: we first show that $\bar p$ is a softmax over the logits
restricted to top-$k$ tokens $S$, then differentiate the loss through it, and finally sum the resulting gradients over $S$.

\paragraph{Step 1: $\bar p$ can be expressed as a softmax restricted to top-$k$ tokens.}
Writing the student distribution as $p(v)=e^{z_v}/\sum_{u\in\mathcal V}e^{z_u}$, the
vocabulary-level normalizer appears in both the numerator and the denominator of $\bar p$ and
therefore cancels,
\begin{equation}
\bar p(v)
=\frac{p(v)}{\sum_{w\in S}p(w)}
=\frac{e^{z_v}}{\sum_{w\in S}e^{z_w}},
\qquad v\in S .
\label{eq:restricted_softmax}
\end{equation}
Two consequences of Eq.~\eqref{eq:restricted_softmax} drive the whole proof: $\bar p$ is a
softmax over the sub-vector $(z_w)_{w\in S}$, and it does not depend on any logit outside
$S$. Its Jacobian is therefore the standard softmax Jacobian on $S$ and vanishes elsewhere,
\begin{equation}
\frac{\partial \bar p(u)}{\partial z_v}
=
\begin{cases}
\bar p(u)\bigl(\mathbf{1}[u=v]-\bar p(v)\bigr), & u,v\in S,\\[0.4em]
0, & v\notin S .
\end{cases}
\label{eq:softmax_jac}
\end{equation}

\paragraph{Step 2: the gradient with respect to logit.}
Because $\bar q$ is fixed with respect to the student logits, the loss depends on the student
only through $\{\bar p(u)\}_{u\in S}$, with
\[
\frac{\partial \ell^{\mathrm{norm}}}{\partial \bar p(u)}
=\log\frac{\bar p(u)}{\bar q(u)}+1,
\qquad u\in S .
\]
The chain rule over the vocabulary thus gives
\begin{equation}
\frac{\partial \ell^{\mathrm{norm}}}{\partial z_v}
=\sum_{u\in S}
\frac{\partial \ell^{\mathrm{norm}}}{\partial \bar p(u)} \times
\frac{\partial \bar p(u)}{\partial z_v}
=\sum_{u\in S}
\left(\log\frac{\bar p(u)}{\bar q(u)}+1\right)
\frac{\partial \bar p(u)}{\partial z_v},
\label{eq:chain_rule}
\end{equation}
where the terms with $u\notin S$ are already dropped because $\ell^{\mathrm{norm}}$ does not
involve them.
For $v\notin S$, every Jacobian entry in Eq.~\eqref{eq:softmax_jac} is zero, hence
$\partial \ell^{\mathrm{norm}}/\partial z_v=0$, which is the second case of the claim: the
loss is blind to the logits outside the teacher's top-$k$ tokens.
For $v\in S$, substituting Eq.~\eqref{eq:softmax_jac} into Eq.~\eqref{eq:chain_rule} and
separating the diagonal term $u=v$ gives
\begin{equation}
\frac{\partial \ell^{\mathrm{norm}}}{\partial z_v}
=\bar p(v)\left(\log\frac{\bar p(v)}{\bar q(v)}+1\right)
-\bar p(v)\sum_{u\in S}\bar p(u)\left(\log\frac{\bar p(u)}{\bar q(u)}+1\right).
\label{eq:logit_grad}
\end{equation}
Note that
\[
\sum_{u\in S}\bar p(u)\left(\log\frac{\bar p(u)}{\bar q(u)}+1\right)
=\sum_{u\in S}\bar p(u)\log\frac{\bar p(u)}{\bar q(u)}+\sum_{u\in S}\bar p(u)
=\ell^{\mathrm{norm}}+1.
\]
By substituding it into Eq~\ref{eq:logit_grad}, we obtain
\begin{equation}
\frac{\partial \ell^{\mathrm{norm}}}{\partial z_v}
=\bar p(v)\left[\log\frac{\bar p(v)}{\bar q(v)}-\ell^{\mathrm{norm}}\right],
\qquad v\in S,
\label{eq:topk_grad}
\end{equation}
which is the first case of the claim. Since
$\ell^{\mathrm{norm}}=\mathbb E_{v\sim\bar p}\bigl[\log(\bar p(v)/\bar q(v))\bigr]$ by
definition, Eq.~\eqref{eq:topk_grad} says that the gradient on a top-$k$ logit is its
log-ratio centered at the $\bar p$-weighted mean log-ratio.

\paragraph{Step 3: the top-$k$ gradients sum to zero.}
Summing Eq.~\eqref{eq:topk_grad} over $v\in S$ and using
$\sum_{v\in S}\bar p(v)\log\bigl(\bar p(v)/\bar q(v)\bigr)=\ell^{\mathrm{norm}}$ together with
$\sum_{v\in S}\bar p(v)=1$ yields
\[
\sum_{v\in S}\frac{\partial \ell^{\mathrm{norm}}}{\partial z_v}
=\ell^{\mathrm{norm}}-\ell^{\mathrm{norm}}
=0 .
\]
This is precisely the mean-centering exhibited in Eq.~\eqref{eq:topk_grad}: the update
redistributes logit mass among the top-$k$ tokens without changing their total, while
leaving the tail logits untouched by Step 2. The normalized objective therefore acts only on
the shape of the student distribution within $S$, and never on how much probability the
student assigns to $S$ as a whole.
\end{proof}

\subsection{Proof of Proposition~\ref{prop:tail_prob_increase}}
\label{app:tail_prob_increase}

\begin{proof}
Fix a step $t$ and a prefix $\hat y_{<t}$, and omit the subscript $t$ throughout, following
the convention of Appendix~\ref{app:topk_gradient}: write $S=S_t^k$, $z_v=z_{t,v}$,
$\bar p=\bar p_t^{(S_t^k)}$, $\bar q=\bar q_t^{(S_t^k)}$, and
$p^{\mathrm{tail}}=p_t^{\mathrm{tail}}$.
Under the tabular assumption, each logit $z_v$ is an independent parameter, so the gradient
step acts on the logits directly; this assumption is commonly adopted by prior works for
analysis~\citep{cui2025entropy, zhang2026characterizing, ren2025learning}.
We write $p^{\mathrm{tail}}(\eta)$ for the tail probability after one gradient step with
learning rate $\eta$, and $\Delta p^{\mathrm{tail}}=p^{\mathrm{tail}}(\eta)-p^{\mathrm{tail}}(0)$
for its change.
We proceed in three steps: we first reduce $p^{\mathrm{tail}}(\eta)$ to a one-dimensional
function of the step size, then identify the derivative of that function as a covariance,
and finally expand it to first order in $\eta$.

\paragraph{Step 1: only the top-$k$ logits change.}
Split the softmax normalizer into the contributions of the top-$k$ tokens and of the tail
tokens,
\[
A:=\sum_{v\in S}e^{z_v},\qquad
B:=\sum_{v\notin S}e^{z_v},\qquad
Z:=A+B,
\]
so that $p^{\mathrm{tail}}=B/Z$.
By Proposition~\ref{prop:topk_gradient}, the gradient of $\ell^{\mathrm{norm}}$ is
\[
g_v=\bar p(v)\left(\log\frac{\bar p(v)}{\bar q(v)}-\ell^{\mathrm{norm}}\right)
\ \ \text{for } v\in S,
\qquad
g_v=0 \ \ \text{for } v\notin S .
\]
The update $z_v\leftarrow z_v-\eta g_v$ therefore leaves tail logit unchanged, and
hence leaves $B$ unchanged. 
The change of the tail probability is thus determined by $A$ alone:
\begin{equation}
p^{\mathrm{tail}}(\eta)=\frac{B}{A(\eta)+B},
\qquad
A(\eta)=\sum_{v\in S}e^{z_v-\eta g_v}.
\label{eq:tail_of_eta}
\end{equation}

\paragraph{Step 2: the derivative of $A$ is a covariance.}
Differentiating $A(\eta)$ in Eq.~\eqref{eq:tail_of_eta} with respect to $\eta$ and evaluating
at $\eta=0$ gives $$A'(0)=-\sum_{v\in S}e^{z_v}g_v.$$
We now rewrite this sum in terms of the normalized distribution $\bar p$.
Using $e^{z_v}=Z\,p(v)$ and $p(v)=(1-p^{\mathrm{tail}})\,\bar p(v)$ for $v\in S$, and then
substituting the expression for $g_v$ from Step 1,
\begin{equation}
A'(0)
=-Z\bigl(1-p^{\mathrm{tail}}\bigr)\sum_{v\in S}\bar p(v)\,g_v
=-Z\bigl(1-p^{\mathrm{tail}}\bigr)
\left[
\sum_{v\in S}\bar p(v)^2\log\frac{\bar p(v)}{\bar q(v)}
-\ell^{\mathrm{norm}}\sum_{v\in S}\bar p(v)^2
\right].
\label{eq:A_prime}
\end{equation}
Each of the two sums in Eq.~\eqref{eq:A_prime} carries one factor $\bar p(v)$ that plays the
role of a sampling weight, so both are expectations under $v\sim\bar p$:
\[
\sum_{v\in S}\bar p(v)^2\log\frac{\bar p(v)}{\bar q(v)}
=\E_{v\sim\bar p}\!\left[\bar p(v)\log\frac{\bar p(v)}{\bar q(v)}\right],
\qquad
\sum_{v\in S}\bar p(v)^2=\E_{v\sim\bar p}\bigl[\bar p(v)\bigr].
\]
Moreover $\ell^{\mathrm{norm}}=\E_{v\sim\bar p}\bigl[\log(\bar p(v)/\bar q(v))\bigr]$ by
definition, so the bracket in Eq.~\eqref{eq:A_prime} is the difference between the
expectation of a product and the product of the expectations, which is exactly a covariance:
\begin{equation}
A'(0)
=-Z\bigl(1-p^{\mathrm{tail}}\bigr)
\operatorname*{Cov}_{v\sim\bar p}\!\left(\bar p(v),\,\log\frac{\bar p(v)}{\bar q(v)}\right).
\label{eq:A_prime_cov}
\end{equation}
The covariance appears because, by Proposition~\ref{prop:topk_gradient}, the gradient on a
top-$k$ logit is its log-ratio centered at the mean log-ratio, while its contribution to $A$
is weighted by its own probability $\bar p(v)$.

\paragraph{Step 3: first-order taylor expansion.}
Differentiating $p^{\mathrm{tail}}(\eta)$ in Eq.~\eqref{eq:tail_of_eta} at $\eta=0$ and
substituting Eq.~\eqref{eq:A_prime_cov} together with $B/Z=p^{\mathrm{tail}}$ yields
\[
\frac{\mathrm{d}p^{\mathrm{tail}}}{\mathrm{d}\eta}\bigg|_{\eta=0}
=-\frac{B}{(A+B)^2}\,A'(0)
=p^{\mathrm{tail}}\bigl(1-p^{\mathrm{tail}}\bigr)
\operatorname*{Cov}_{v\sim\bar p}\!\left(\bar p(v),\,\log\frac{\bar p(v)}{\bar q(v)}\right).
\]
A first-order Taylor expansion in $\eta$, followed by splitting the log-ratio as
$\log\bar p(v)-\log\bar q(v)$ and using the linearity of the covariance in its second
argument, gives the stated identity:
\[
\Delta p^{\mathrm{tail}}
=\eta\,p^{\mathrm{tail}}\bigl(1-p^{\mathrm{tail}}\bigr)
\Bigl[
\operatorname*{Cov}_{v\sim\bar p}\bigl(\bar p(v),\log\bar p(v)\bigr)
-\operatorname*{Cov}_{v\sim\bar p}\bigl(\bar p(v),\log\bar q(v)\bigr)
\Bigr]
+O(\eta^2).
\]
Since $p^{\mathrm{tail}}\in(0,1)$, the prefactor $p^{\mathrm{tail}}(1-p^{\mathrm{tail}})$ is
strictly positive, so for sufficiently small $\eta>0$ the sign of $\Delta p^{\mathrm{tail}}$
is determined by the sign of the covariance difference. The tail probability therefore
increases if and only if the student self-covariance exceeds the student--teacher covariance.
\end{proof}

\subsection{Proof of Proposition~\ref{prop:ta}}
\label{app:ta}

\begin{proof}
Fix a step $t$ and a prefix $\hat y_{<t}$, and omit the subscript $t$ throughout, following
the convention of Appendix~\ref{app:topk_gradient}: write $S=S_t^k$, $z_v=z_{t,v}$,
$p(v)=p_t(v)$, $q(v)=q_t(v)$, and $p^{\mathrm{tail}}=p_t^{\mathrm{tail}}$,
$q^{\mathrm{tail}}=q_t^{\mathrm{tail}}$, so that
\[
\ell^{\mathrm{TA}}
=\sum_{u\in S}p(u)\log\frac{p(u)}{q(u)}
+p^{\mathrm{tail}}\log\frac{p^{\mathrm{tail}}}{q^{\mathrm{tail}}},
\qquad
p^{\mathrm{tail}}=\sum_{u\notin S}p(u).
\]
Since $S$ is selected from the teacher distribution, $S$, $q$, and $q^{\mathrm{tail}}$ are
constants with respect to the student logits.
Unlike the normalized loss of Proposition~\ref{prop:topk_gradient}, $\ell^{\mathrm{TA}}$
depends on the student through the \emph{unnormalized} probabilities, and the tail tokens
enter it through their aggregate $p^{\mathrm{tail}}$.
We proceed in three steps: we first differentiate $\ell^{\mathrm{TA}}$ with respect to the
probabilities of all tokens, then push the derivative through the full softmax Jacobian, and
finally read off the two cases of the claim.

\paragraph{Step 1: the derivative with respect to probability.}
We treat $\ell^{\mathrm{TA}}$ as a function of the full probability vector
$\bigl(p(u)\bigr)_{u\in\mathcal V}$ rather than eliminating $p^{\mathrm{tail}}$.
A top-$k$ token $u\in S$ appears only in its own summand, and a tail token $u\notin S$
appears only inside $p^{\mathrm{tail}}$, on which it acts with
$\partial p^{\mathrm{tail}}/\partial p(u)=1$. Hence
\begin{equation}
\frac{\partial \ell^{\mathrm{TA}}}{\partial p(u)}
=
\begin{cases}
\log\dfrac{p(u)}{q(u)}+1, & u\in S,\\[1.0em]
\log\dfrac{p^{\mathrm{tail}}}{q^{\mathrm{tail}}}+1, & u\notin S .
\end{cases}
\label{eq:ta_dp}
\end{equation}

\paragraph{Step 2: pushing through the softmax Jacobian.}
The student distribution is a softmax over the full vocabulary,
$p(u)=e^{z_u}/\sum_{w\in\mathcal V}e^{z_w}$, whose Jacobian is
$\partial p(u)/\partial z_v=p(u)\bigl(\mathbf 1[u=v]-p(v)\bigr)$ for all
$u,v\in\mathcal V$. Combining it with Eq.~\eqref{eq:ta_dp} through the chain rule gives
\begin{equation}
\frac{\partial \ell^{\mathrm{TA}}}{\partial z_v}
=\sum_{u\in\mathcal V}
\frac{\partial \ell^{\mathrm{TA}}}{\partial p(u)}\,
\frac{\partial p(u)}{\partial z_v}
=p(v)\,\frac{\partial \ell^{\mathrm{TA}}}{\partial p(u)}\bigg|_{u=v}
-p(v)\sum_{u\in\mathcal V}p(u)\,\frac{\partial \ell^{\mathrm{TA}}}{\partial p(u)} .
\label{eq:ta_chain}
\end{equation}
Note that for the second term, we have
\[
\sum_{u\in\mathcal V}p(u)\,\frac{\partial \ell^{\mathrm{TA}}}{\partial p(u)}
=\sum_{u\in S}p(u)\left(\log\frac{p(u)}{q(u)}+1\right)
+\left(\sum_{u\notin S}p(u)\right)\!\left(\log\frac{p^{\mathrm{tail}}}{q^{\mathrm{tail}}}+1\right)
=\ell^{\mathrm{TA}}+1,
\]
where the second equality uses $\sum_{u\notin S}p(u)=p^{\mathrm{tail}}$ to recover the tail
term of $\ell^{\mathrm{TA}}$, and $\sum_{u\in\mathcal V}p(u)=1$ to collect the constant.

\paragraph{Step 3: the gradient on each logit.}
Substituting $\ell^{\mathrm{TA}}+1$ for the sum in Eq.~\eqref{eq:ta_chain}, we have
\begin{equation}
\frac{\partial \ell^{\mathrm{TA}}}{\partial z_v}
=
\begin{cases}
p(v)\left(\log\dfrac{p(v)}{q(v)}-\ell^{\mathrm{TA}}\right), & v\in S,\\[1.2em]
p(v)\left(\log\dfrac{p^{\mathrm{tail}}}{q^{\mathrm{tail}}}-\ell^{\mathrm{TA}}\right), & v\notin S,
\end{cases}
\label{eq:ta_grad}
\end{equation}
which establishes \eqref{eq:gma}.
\end{proof}

\subsection{Proof of Proposition~\ref{prop:bound}}
\label{app:bound}
\begin{proof}
Fix a step $t$ and drop the subscript $t$ as in Appendix~\ref{app:ta}.
The two loss functions are
\[
\ell^{\mathrm{full}}
=
\sum_{v\in S}p(v)\log\frac{p(v)}{q(v)}
+
\sum_{v\notin S}p(v)\log\frac{p(v)}{q(v)},
\qquad
\ell^{\mathrm{TA}}
=
\sum_{v\in S}p(v)\log\frac{p(v)}{q(v)}
+
p_{\mathrm{tail}}\log\frac{p_{\mathrm{tail}}}{q_{\mathrm{tail}}},
\]
where we have split $\ell^{\mathrm{full}}=\DKL(p\|q)$ into the top-$k$ tokens and the
tail tokens so that the two losses can be compared term by term.
The two objectives share the same top-$k$ term and differ only in how they treat the
tail, so taking the difference cancels the top-$k$ term and leaves
\begin{equation}
\ell^{\mathrm{full}}-\ell^{\mathrm{TA}}
=
\sum_{v\notin S}p(v)\log\frac{p(v)}{q(v)}
-
p_{\mathrm{tail}}\log\frac{p_{\mathrm{tail}}}{q_{\mathrm{tail}}}.
\label{eq:bound_diff}
\end{equation}

It remains to compute the right-hand side of \eqref{eq:bound_diff}. Substituting
$p(v)=p_{\mathrm{tail}}\,\tilde p(v)$ and $q(v)=q_{\mathrm{tail}}\,\tilde q(v)$ for
$v\notin S$ into its first term and splitting the logarithm gives
\[
\begin{aligned}
\sum_{v\notin S}p(v)\log\frac{p(v)}{q(v)}
&=
\sum_{v\notin S}
p_{\mathrm{tail}}\,\tilde p(v)
\log\frac{p_{\mathrm{tail}}\,\tilde p(v)}{q_{\mathrm{tail}}\,\tilde q(v)} \\
&=
p_{\mathrm{tail}}
\log\frac{p_{\mathrm{tail}}}{q_{\mathrm{tail}}}
\sum_{v\notin S}\tilde p(v)
+
p_{\mathrm{tail}}
\sum_{v\notin S}\tilde p(v)\log\frac{\tilde p(v)}{\tilde q(v)} \\
&=
p_{\mathrm{tail}}
\log\frac{p_{\mathrm{tail}}}{q_{\mathrm{tail}}}
+
p_{\mathrm{tail}}\,\DKL(\tilde p\|\tilde q),
\end{aligned}
\]
where the last equality uses that $\tilde p$ is a probability distribution over the tail
tokens, \ $\sum_{v\notin S}\tilde p(v)=1$, and recognizes the remaining sum as
$\DKL(\tilde p\|\tilde q)$.

Substituting this back into \eqref{eq:bound_diff}, the terms
$p_{\mathrm{tail}}\log\frac{p_{\mathrm{tail}}}{q_{\mathrm{tail}}}$ cancel and we obtain
\[
\ell^{\mathrm{full}}-\ell^{\mathrm{TA}}
=
p_{\mathrm{tail}}\,\DKL(\tilde p\|\tilde q).
\]
\end{proof}

\subsection{Proof of Proposition~\ref{prop:unbias}}
\label{app:unbias}

\begin{proof}
Fix step $t$ and drop the subscript $t$ as in Appendix~\ref{app:ta}: write $S=S_t^k$,
$S^{+}=S_t^{+}$, $p(v)=p_t(v)$, $q(v)=q_t(v)$, $\hat y=\hat y_t$,
$p_{\mathrm{tail}}=p_t^{\mathrm{tail}}$, and $q_{\mathrm{tail}}=q_t^{\mathrm{tail}}$.
Recall that the stop-gradient operator $\operatorname{sg}(\cdot)$ acts as the identity in the
forward pass and has zero derivative, i.e., $\operatorname{sg}(x)=x$ in value and
$\nabla_\theta \operatorname{sg}(x)=0$.
Throughout, the expectation is taken with respect to $\hat y\sim p$, where the sampling
distribution is fixed at the current parameters.
We will repeatedly use the identity
\begin{equation}
\E_{\hat y\sim p}\bigl[\mathbf{1}[\hat y\notin S]\bigr]=\sum_{v\notin S}p(v)=p_{\mathrm{tail}}.
\label{eq:tail_indicator}
\end{equation}

\paragraph{Unbiasedness in value.}
Without considering the gradient, $\frac{p(\hat y)}{\operatorname{sg}(p(\hat y))}=1$, so
\begin{align}
\ell^{\mathrm{SC\text{-}TA}}
&= \sum_{v\in S^{+}} p(v)\operatorname{sg}\!\left(\log\frac{p(v)}{q(v)}\right)
+ \mathbf{1}[\hat y\notin S]\,
\frac{p(\hat y)}{\operatorname{sg}(p(\hat y))}
\operatorname{sg}\!\left(\log\frac{p(\hat y)/p_{\mathrm{tail}}}{q(\hat y)/q_{\mathrm{tail}}}\right)
\notag \\
&= \sum_{v\in S} p(v)\log\frac{p(v)}{q(v)}
+ p_{\mathrm{tail}}\log\frac{p_{\mathrm{tail}}}{q_{\mathrm{tail}}}
+ \mathbf{1}[\hat y\notin S]\left(
\log\frac{p(\hat y)}{q(\hat y)}-\log\frac{p_{\mathrm{tail}}}{q_{\mathrm{tail}}}
\right),
\end{align}
where the second equality expands $S^{+}=S\cup\{v_{\mathrm{tail}}\}$ using
$p(v_{\mathrm{tail}})=p_{\mathrm{tail}}$ and $q(v_{\mathrm{tail}})=q_{\mathrm{tail}}$.
The first two terms are deterministic given the prefix.
For the third term, taking the expectation over $\hat y\sim p$ and using
Eq.~\eqref{eq:tail_indicator} gives
\[
\E_{\hat y\sim p}\!\left[
\mathbf{1}[\hat y\notin S]\left(
\log\frac{p(\hat y)}{q(\hat y)}-\log\frac{p_{\mathrm{tail}}}{q_{\mathrm{tail}}}
\right)\right]
=\sum_{v\notin S}p(v)\log\frac{p(v)}{q(v)}
-p_{\mathrm{tail}}\log\frac{p_{\mathrm{tail}}}{q_{\mathrm{tail}}}.
\]
Summing the three terms, the two occurrences of
$p_{\mathrm{tail}}\log\frac{p_{\mathrm{tail}}}{q_{\mathrm{tail}}}$ cancel and we obtain
\[
\E_{\hat y\sim p}\bigl[\ell^{\mathrm{SC\text{-}TA}}\bigr]
=\sum_{v\in S}p(v)\log\frac{p(v)}{q(v)}
+\sum_{v\notin S}p(v)\log\frac{p(v)}{q(v)}
=\DKL(p\,\|\,q),
\]
which establishes the unbiasedness in value.

\paragraph{Unbiasedness in gradient.}
We first compute the gradient of the exact full-vocabulary reverse KL. Since $q$ is fixed,
\[
\nabla_\theta\,\DKL(p\,\|\,q)
=\sum_{v\in\mathcal V}\nabla_\theta p(v)\left(\log\frac{p(v)}{q(v)}+1\right)
=\sum_{v\in\mathcal V}\nabla_\theta p(v)\,\log\frac{p(v)}{q(v)},
\]
where the last equality uses
$\sum_{v\in\mathcal V}\nabla_\theta p(v)=\nabla_\theta\sum_{v\in\mathcal V}p(v)=\nabla_\theta 1=0$.

We now differentiate $\ell^{\mathrm{SC\text{-}TA}}$.
Since the log-ratios are wrapped in stop-gradients, only the leading probabilities carry
gradients in the first term.
Note that $p_{\mathrm{tail}}=\sum_{v\notin S}p(v)$ depends on $\theta$, so
$\nabla_\theta p_{\mathrm{tail}}=\sum_{v\notin S}\nabla_\theta p(v)$, and therefore
\[
\nabla_\theta\sum_{v\in S^{+}}p(v)\operatorname{sg}\!\left(\log\frac{p(v)}{q(v)}\right)
=\sum_{v\in S}\nabla_\theta p(v)\,\log\frac{p(v)}{q(v)}
+\Bigl(\sum_{v\notin S}\nabla_\theta p(v)\Bigr)
\log\frac{p_{\mathrm{tail}}}{q_{\mathrm{tail}}}.
\]
For the second term, only the numerator $p(\hat y)$ in
$\frac{p(\hat y)}{\operatorname{sg}(p(\hat y))}$ carries gradients, so
\[
\nabla_\theta\!\left[
\mathbf{1}[\hat y\notin S]\,
\frac{p(\hat y)}{\operatorname{sg}(p(\hat y))}
\operatorname{sg}\!\left(\log\frac{p(\hat y)/p_{\mathrm{tail}}}{q(\hat y)/q_{\mathrm{tail}}}\right)
\right]
=\mathbf{1}[\hat y\notin S]\,
\frac{\nabla_\theta p(\hat y)}{p(\hat y)}
\left(\log\frac{p(\hat y)}{q(\hat y)}-\log\frac{p_{\mathrm{tail}}}{q_{\mathrm{tail}}}\right).
\]
Taking the expectation over $\hat y\sim p$ of the second term gives
\[
\begin{aligned}
&\E_{\hat y\sim p}\!\left[
\mathbf{1}[\hat y\notin S]\,
\frac{\nabla_\theta p(\hat y)}{p(\hat y)}
\left(\log\frac{p(\hat y)}{q(\hat y)}-\log\frac{p_{\mathrm{tail}}}{q_{\mathrm{tail}}}\right)
\right] \\
&\qquad =\sum_{v\notin S}p(v)\cdot\frac{\nabla_\theta p(v)}{p(v)}
\left(\log\frac{p(v)}{q(v)}-\log\frac{p_{\mathrm{tail}}}{q_{\mathrm{tail}}}\right)
=\sum_{v\notin S}\nabla_\theta p(v)\,\log\frac{p(v)}{q(v)}
-\Bigl(\sum_{v\notin S}\nabla_\theta p(v)\Bigr)
\log\frac{p_{\mathrm{tail}}}{q_{\mathrm{tail}}}.
\end{aligned}
\]
Combining the two terms, the contributions involving
$\log\frac{p_{\mathrm{tail}}}{q_{\mathrm{tail}}}$ cancel exactly and we obtain
\[
\E_{\hat y\sim p}\bigl[\nabla_\theta\,\ell^{\mathrm{SC\text{-}TA}}\bigr]
=\sum_{v\in S}\nabla_\theta p(v)\,\log\frac{p(v)}{q(v)}
+\sum_{v\notin S}\nabla_\theta p(v)\,\log\frac{p(v)}{q(v)}
=\nabla_\theta\,\DKL(p\,\|\,q),
\]
which establishes the unbiasedness in gradient. This completes the proof.
\end{proof}

\section{Detailed Analysis of the tail probability increase}
\label{sec:additional_analysis}

In this section, we conduct several controlled experiments to identify when the
tail probability increase occurs.
We find that it arises under two conditions: a large capability gap between the student and teacher, and an insufficiently large $k$.

\begin{figure}
    \centering
    \includegraphics[width=1.0\linewidth]{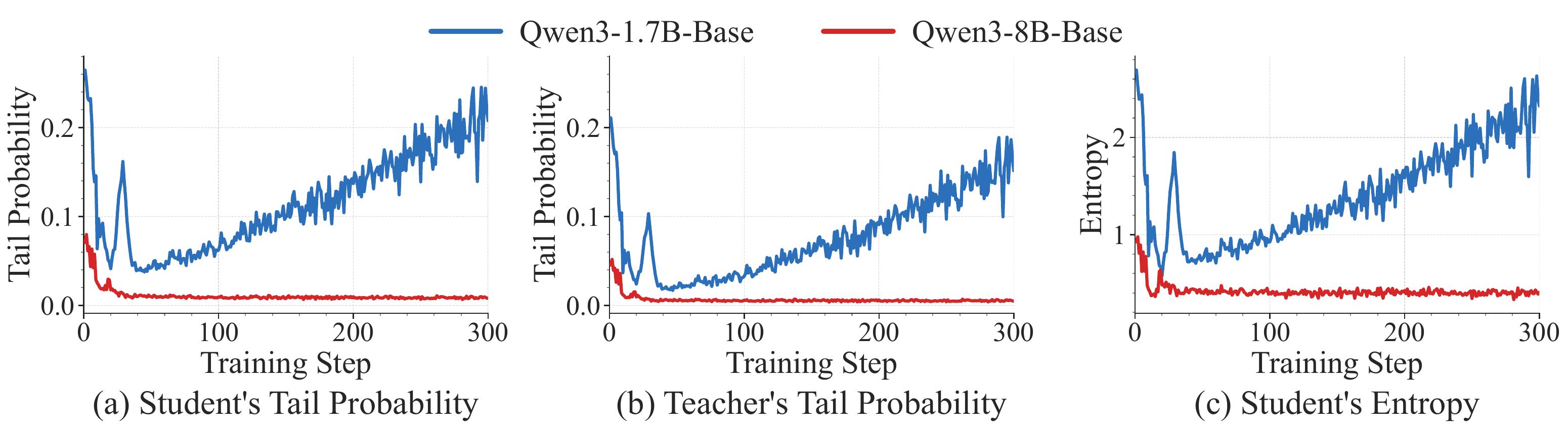}
    \caption{
\textbf{The student's tail probability, the teacher's tail probability, and the student's entropy across training steps under different student--teacher capability gaps.}
The teacher is Qwen3-8B, and the students are Qwen3-1.7B-Base and Qwen3-8B-Base.
The tail probability increase occurs for Qwen3-1.7B-Base but not for Qwen3-8B-Base.
}
    \label{fig:ablate_gap}
\end{figure}

\paragraph{The tail probability increase occurs when the student--teacher capability gap is large.}
We conduct an ablation study to demonstrate that the tail probability increase occurs when the capability gap between the student and teacher is large.
Specifically, we fix Qwen3-8B as the teacher and compare two students, Qwen3-1.7B-Base and Qwen3-8B-Base, where the latter has a smaller capability gap to the teacher.
For this ablation study, we set the maximum response length to 4096.

In Figure~\ref{fig:ablate_gap}, we visualize the student's and teacher's tail probabilities on student-generated prefixes, together with the student's entropy.
The tail probability increase emerges only under the large capability gap: for Qwen3-1.7B-Base, the student's tail probability and entropy steadily increase over training steps, whereas for Qwen3-8B-Base, both remain low and stable.
Overall, the tail probability increase arises when the student--teacher capability gap is large and is mitigated as the gap narrows.


\begin{figure}
    \centering
    \includegraphics[width=1.0\linewidth]{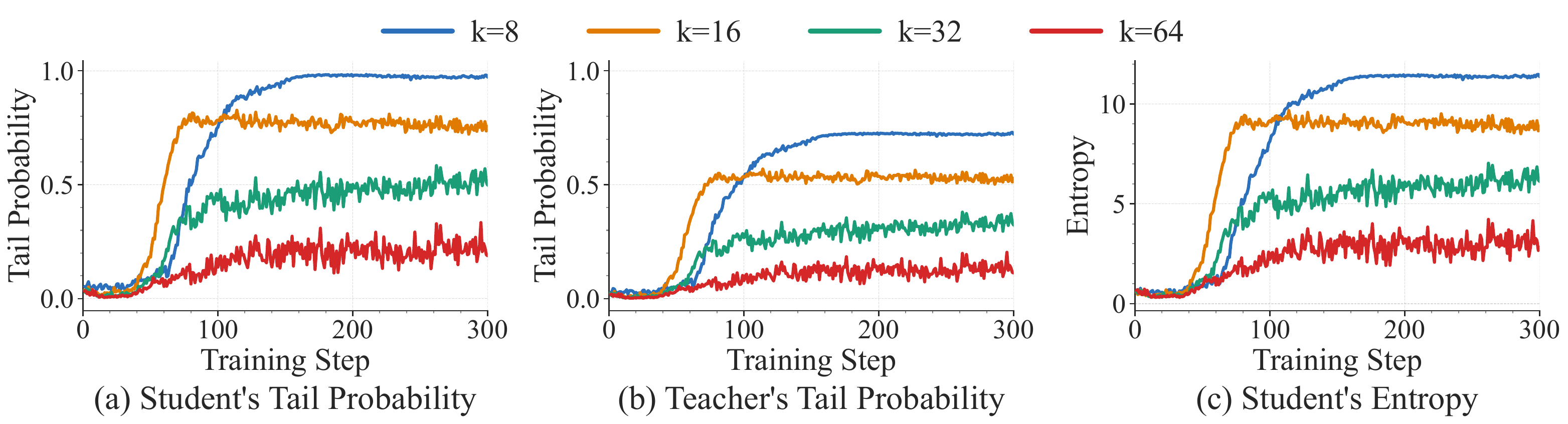}
    \caption{
\textbf{The student's tail probability, the teacher's tail probability, and the student's entropy across training steps.}
The student and teacher models are Qwen2.5-Math-1.5B and DeepSeek-R1-Distill-Qwen-7B, respectively.
The tail probability increase occurs when $k\leq 64$.
}
    \label{fig:ablate_k}
\end{figure}

\paragraph{The tail probability increase occurs when $k$ is not sufficiently large.}
We conduct an ablation study on how $k$ affects the tail probability increase by fixing
the student and teacher, and varying only $k$. Specifically, we use
Qwen2.5-Math-1.5B as the student and DeepSeek-R1-Distill-Qwen-7B as the teacher,
and sweep $k \in \{8, 16, 32, 64\}$. 
For this ablation study, we set the training max response length to 4096.

In Figure~\ref{fig:ablate_k}, we visualize the student's and teacher's tail probability on student-generated prefixes, together with the student's entropy, across different values of $k$. 
The tail probability increase emerges when $k$ is not sufficiently large ($k \leq 64$): over training steps, the student's tail probability steadily increases together with its entropy. 
The effect becomes stronger as $k$ decreases, with smaller $k$ yielding higher entropy and a larger increase in
tail probability. 
In principle, this issue vanishes as $k$ approaches the vocabulary size, since the top-$k$ objective then reduces to the full-vocabulary OPD, under which the tail probability is explicitly matched to the teacher's.
This behavior is expected from Proposition~\ref{prop:tail_prob_increase}: the magnitude of
the single-step increase is proportional to $p_t^{\mathrm{tail}}(1-p_t^{\mathrm{tail}})$,
which shrinks as $p_t^{\mathrm{tail}}\to 0$.
Since a larger $k$ yields a smaller initial tail probability, each update moves
$p_t^{\mathrm{tail}}$ less, mitigating the compounding increase over training.
Overall, the tail probability increase arises when $k$ is not sufficiently large and is mitigated as $k$ increases.

\begin{figure}
    \centering
    \includegraphics[width=1.0\linewidth]{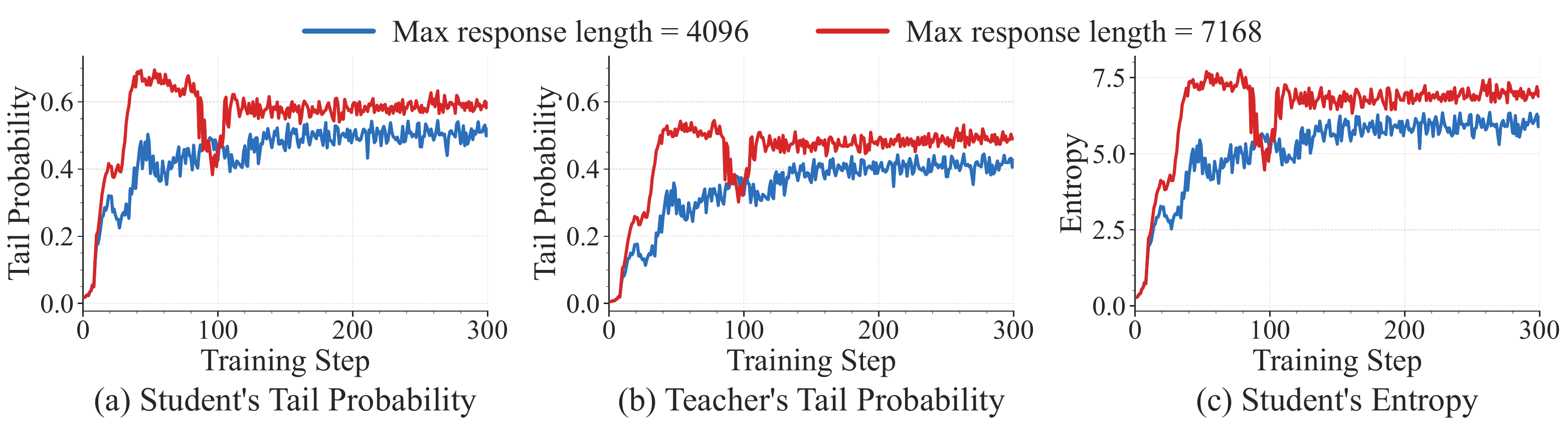}
    \caption{
\textbf{The student's tail probability, the teacher's tail probability, and the student's entropy across training steps under different maximum response lengths.}
The student and teacher models are Qwen2.5-7B-Instruct and OpenThinker3-7B, respectively.
The tail probability increase becomes more pronounced as the maximum response length grows.
}
    \label{fig:ablate_length}
\end{figure}

\paragraph{The tail probability increase becomes more pronounced with longer responses.}
We study how the response length affects the tail probability increase by fixing the student and teacher, and varying only the maximum response length.
Specifically, we use Qwen2.5-7B-Instruct as the student and OpenThinker3-7B as the teacher, and compare maximum response lengths of 4096 and 7168 tokens.

In Figure~\ref{fig:ablate_length}, we visualize the student's and teacher's tail probabilities on student-generated prefixes, together with the student's entropy, under the two length settings.
The results show that the tail probability increase is amplified under the longer response length: the student's tail probability and entropy grow faster and reach higher values throughout training.
We attribute this to error compounding along longer trajectories: as the student's tail probability rises, sampling more tokens per response increases the chance of drifting outside the teacher's top-$k$ tokens, driving the student toward uncertain prefixes where the teacher's supervision is unreliable and further inflating the tail probability.
Overall, the tail probability increase is exacerbated by longer response lengths, suggesting that the issue is particularly concerning for long-horizon tasks.

\paragraph{The tail probability increase does not occur with the student top-$k$ OPD.}
Prior work also adopts the student's top-$k$ tokens as the support~\citep{li2026rethinking}.
We examine this choice with Qwen3-1.7B-Base as the student and Qwen3-8B as the teacher,
varying only the support between $\operatorname{TopK}(q_t,k)$ and $\operatorname{TopK}(p_t,k)$.
In each case, the tail probabilities of both models are computed with respect to the
support in use, i.e.\ the teacher's top-$k$ tokens for the teacher top-$k$ variant and the
student's top-$k$ tokens for the student top-$k$ variant.

As shown in Figure~\ref{fig:ablate_support}, the tail probability increase is specific to
the teacher top-$k$ support.
Two mechanisms explain this.
First, among all token sets of size $k$, $\operatorname{TopK}(p_t,k)$ is the one that leaves
the least probability mass outside it, so the student's tail probability starts near zero.
By Proposition~\ref{prop:tail_prob_increase}, the per-step change is proportional to
$p_t^{\mathrm{tail}}(1-p_t^{\mathrm{tail}})$, and each update therefore moves
$p_t^{\mathrm{tail}}$ only slightly.
Second, the student top-$k$ support is changing given a prefix: it is recomputed from the student at every step, so a token whose probability is pushed down simply drops out of the support and is replaced by another token that the student now ranks in its top $k$.
The support therefore always consists of the student's $k$ most probable tokens, and the
probability mass it covers cannot leak away.
The teacher top-$k$ support, in contrast, is fixed by the teacher and does not follow the
student, so the mass pushed outside it stays outside and accumulates across steps.

\begin{figure}
    \centering
    \includegraphics[width=1.0\linewidth]{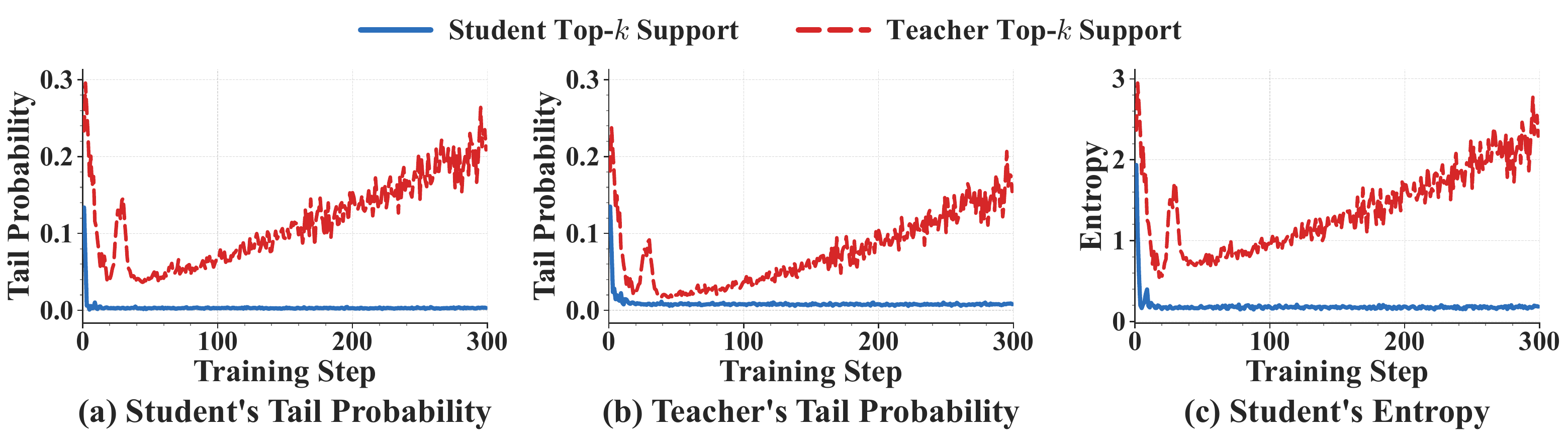}
    \caption{
\textbf{The student's tail probability, the teacher's tail probability, and the student's
entropy across training steps under the student and teacher top-$k$ supports.}
The student and teacher models are Qwen3-1.7B-Base and Qwen3-8B.
The tail probabilities are computed with respect to the support in use.
The tail probability increase occurs only under the teacher top-$k$ support.
}
    \label{fig:ablate_support}
\end{figure}

\paragraph{The teacher top-$k$ OPD yields better downstream performance than the student top-$k$ OPD.}
Although the student top-$k$ support avoids the tail probability increase, it underperforms
the teacher top-$k$ support on downstream benchmarks (Table~\ref{tab:ablate_support}).
We attribute this to two factors.
First, as shown in Figure~\ref{fig:compare_response_length}, the average rollout length under
the student top-$k$ support grows monotonically and saturates at the maximum response length
of $7168$ tokens within the first $50$ steps, while the teacher top-$k$ variant maintains a healthy response length.
Second, the teacher top-$k$ support directly covers the teacher's high-probability tokens, whereas the student's top-$k$ tokens may carry little teacher probability, allowing the objective to be reduced without moving the student toward the teacher's high-probability modes.
Given the superior downstream performance, we mainly study the teacher top-$k$ OPD in this work.

\begin{figure}
    \centering
    \includegraphics[width=0.6\linewidth]{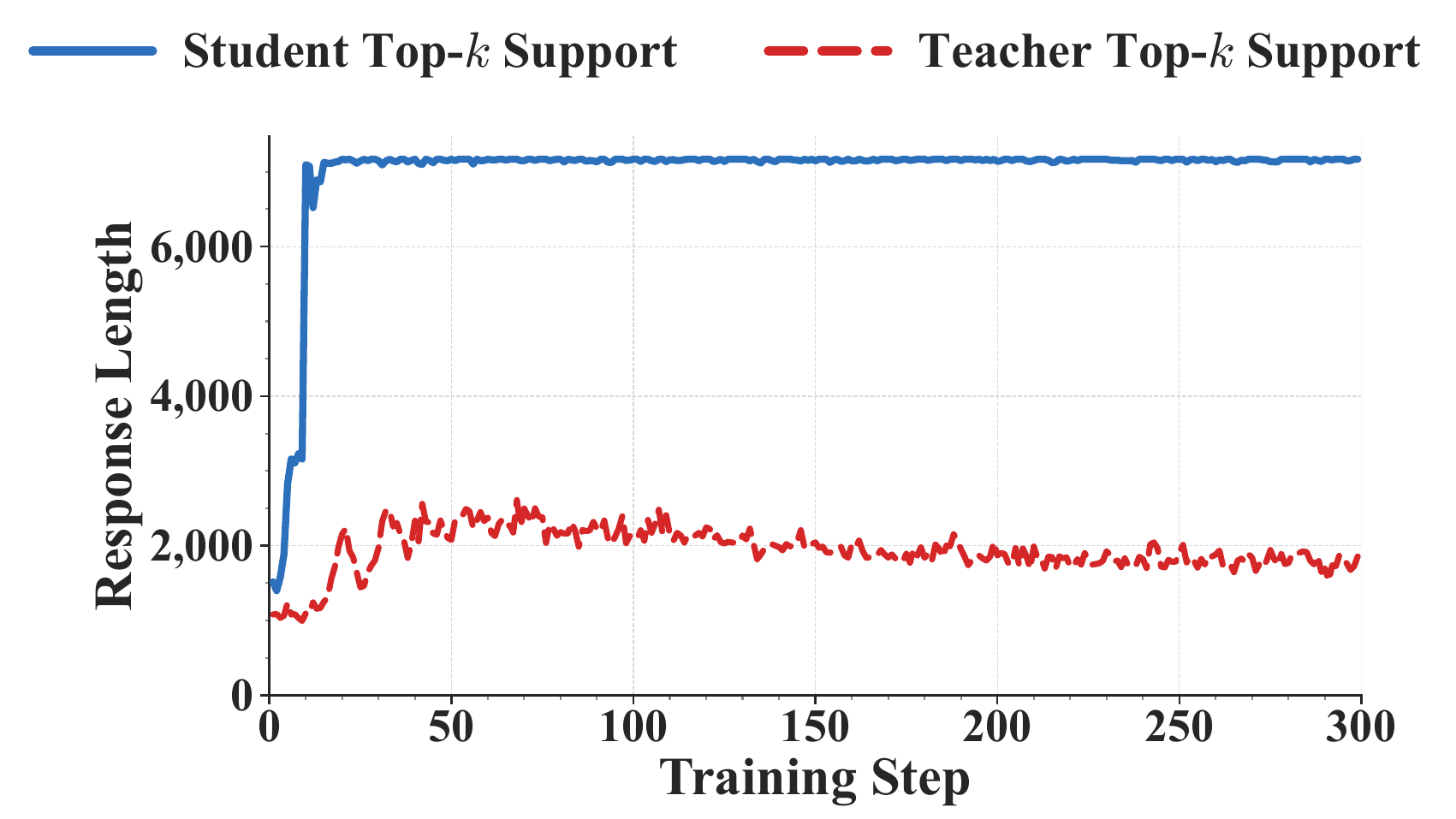}
    \caption{
\textbf{Average response length across training steps under the student and teacher top-$k$
supports.}
The student and teacher models are Qwen3-1.7B-Base and Qwen3-8B, respectively.
Under the student top-$k$ support, the response length saturates at the maximum response
length of $7168$ tokens, while the teacher top-$k$ variant maintains a healthy response length.
}
    \label{fig:compare_response_length}
\end{figure}

\begin{table}[t]
    \centering
    \caption{
    \textbf{Performance of normalized top-$k$ OPD under the student and teacher top-$k$ supports.}
    The \textbf{Avg.} column represents the macro-average across all six math reasoning benchmarks.
    Best results are shown in \textbf{bold}.
    }
    \label{tab:ablate_support}
    \setlength{\tabcolsep}{5pt}
    \renewcommand{\arraystretch}{1.15}
    \resizebox{\textwidth}{!}{
        \begin{tabular}{l|cccccc|c}
        \toprule
        \textbf{Method}
        & \textbf{MATH500}
        & \textbf{Minerva}
        & \textbf{Olympiad}
        & \textbf{AMC}
        & \textbf{AIME 24}
        & \textbf{AIME 25}
        & \textbf{Avg.} \\
        \midrule
        \multicolumn{8}{c}{
            Student: \textbf{\textit{Qwen3-1.7B-Base}}
            \quad
            Teacher: \textbf{\textit{Qwen3-8B}}
        } \\
        \midrule
        Student top-$k$
        & 68.58
        & 25.32
        & 33.35
        & 36.00
        & 9.17
        & 6.25
        & 29.78 \\

        Teacher top-$k$
        & \textbf{69.93}
        & \textbf{25.46}
        & \textbf{33.41}
        & 36.00
        & 9.17
        & \textbf{7.08}
        & \textbf{30.18} \\
        \bottomrule
        \end{tabular}
    }
\end{table}

\section{Compared Methods}
\label{sec:methods}

\paragraph{Sampled-token OPD.}
Sampled-token OPD~\citep{lu2025onpolicydistillation, xiao2026mimo} supervises only the token $\hat y_t\sim p_t$ sampled by the student:
\[
\ell_t^{\mathrm{sample}} = \log p_t(\hat y_t) - \log q_t(\hat y_t).
\]
It is an unbiased estimate of the full-vocabulary reverse KL divergence and requires only the sampled token's log-probability from the teacher, but it discards the dense information over the remaining vocabulary and suffers from high variance.

\paragraph{Unnormalized top-$k$ OPD.}
Unnormalized top-$k$ OPD restricts the divergence computation to the teacher's top-$k$ tokens $S_t^k=\operatorname{TopK}(q_t,k)$ while keeping the original probabilities:
\[
\ell_t^{\mathrm{unnorm}} = \sum_{v \in S_t^k} p_t(v)\log\frac{p_t(v)}{q_t(v)}.
\]
It provides dense supervision over $S_t^k$, but is not a well-defined divergence.

\paragraph{Normalized top-$k$ OPD.}
Normalized top-$k$ OPD~\citep{li2026rethinking, fu2026revisiting} instead renormalizes both distributions on $S_t^k$ and minimizes the resulting subset reverse KL divergence:
\[
\ell_t^{\mathrm{norm}} = D_{\mathrm{KL}}\!\bigl(\bar{p}_t^{(S_t^k)} \,\|\, \bar{q}_t^{(S_t^k)}\bigr) = \sum_{v \in S_t^k} \bar{p}_t^{(S_t^k)}(v)\log\frac{\bar{p}_t^{(S_t^k)}(v)}{\bar{q}_t^{(S_t^k)}(v)}.
\]
It aligns the student's relative shape with the teacher's over $S_t^k$, but the normalization discards the tail probability, which, as we show in Section~\ref{sec:motivation_theory}, steadily increases the student's tail probability and entropy.

\section{Implementation Details}
\label{sec:implementation_detail}
\subsection{Experimental Details} 
We implement all methods using the VERL framework~\citep{sheng2025hybridflow} and conduct experiments with 6 NVIDIA Pro 6000 GPUs, 8 NVIDIA A100 80\,GB GPUs, or 32 NVIDIA A100 40\,GB GPUs.
Unless otherwise specified, all experiments use the default settings and hyperparameters listed in Table~\ref{tab:opd_hyperparameter}.
For evaluation, we adopt the same prompt as~\citet{yan2026learning},
which is shown in Prompt~\ref{prompt:evaluation}.
For experiments using Llama-3.1-8B as the student, we adopt the tokenizer and chat template of DeepSeek-R1-Distill-Llama-8B and synchronize the corresponding vocabulary and special-token configurations.

\begin{table}[t]
\centering
\caption{\textbf{Default training and evaluation settings.}}
\label{tab:opd_hyperparameter}
\setlength{\tabcolsep}{15pt}

\begin{tabular}{llc}
\toprule
\textbf{Category} & \textbf{Item} & \textbf{Value} \\
\midrule
\multirow{11}{*}{Training}
& Training temperature & 1.0 \\
& Global batch size & 72 \\
& Mini batch size & 36 \\
& Rollout number & 4 \\
& $k$ (Number of Top-$k$ Tokens) & 16 \\
& Top-$p$ & 1.0 \\
& Max prompt length & 1024 \\
& Max response length & 7168 \\
& Learning rate & 1e-6 \\
& Training step & 300 \\
& loss aggregation & token-mean \\
& optimizer & AdamW \\
\midrule
\multirow{3}{*}{Evaluation}
& Temperature & 0.7 \\
& Top-$p$ & 0.95 \\
& Max new tokens & 8192 \\
\bottomrule
\end{tabular}
\end{table}

\begin{tcolorbox}[
    center,
    arc=0mm,
    boxrule=1pt,
    colback=blue!6!white,
    colframe=black,
    colbacktitle=black,
    attach boxed title to top left={yshift=-0.1in,xshift=0.15in},
    boxed title style={boxrule=0pt,colframe=white}
]
\label{prompt:evaluation}
Your task is to follow a systematic, thorough reasoning process before providing the final solution. This involves analyzing, summarizing, exploring, reassessing, and refining your thought process through multiple iterations. Structure your response into two sections: Thought and Solution. In the Thought section, present your reasoning using the format: "\verb|<think>\n| {thoughts} \verb|</think>\n|". Each thought should include detailed analysis, brainstorming, verification, and refinement of ideas. After "\verb|</think>\n|" in the Solution section, provide the final, logical, and accurate answer, clearly derived from the exploration in the Thought section. If applicable, include the answer in \verb|\boxed{}| for closed-form results like multiple choices or mathematical solutions. \\
\textbf{User:}  \verb|{QUESTION}| \\
\textbf{Assistant:}
\end{tcolorbox}

\subsection{Numerically Stable Computation of TA-OPD}
\label{subsec:stable_maopd}

The per-token TA-OPD loss in Eq.~\eqref{eq:lma} depends on the tail
log-probabilities $\log p_t^{\mathrm{tail}}$ and $\log q_t^{\mathrm{tail}}$.
Below, we describe a failure case of a naive implementation and how we
implement the loss.

\paragraph{Numerical issue of a naive implementation.}
The inference engine vLLM only provides access to the teacher's top-$k$
log-probabilities $\log q_t(v)$, not the probabilities themselves.
A naive implementation of TA-OPD loss is to exponentiate these log-probabilities back to
probability space, sum them to obtain the total probability of the top-$k$ tokens
$1-q_t^{\mathrm{tail}}=\sum_{v\in S_t^k}q_t(v)$, and assign the remaining tail probability
$q_t^{\mathrm{tail}}$ to $v_{\mathrm{tail}}$.
However, this is unstable.
When the teacher's tail probability is close to zero, numerical error can make the sum $\sum_{v\in S_t^k}q_t(v)$ exceed one.
The tail probability $q_t^{\mathrm{tail}}$ then becomes negative, rendering
$\log q_t^{\mathrm{tail}}$ ill-defined ($-\infty$ or NaN) and producing NaN gradients.

\paragraph{Our log-space implementation.}
To address the numerical issue, our implementation avoids forming $p_t^{\mathrm{tail}}$ and $q_t^{\mathrm{tail}}$ in probability space and carries out the whole
computation in log space.
We provide the pseudo code of our implementation in Listing~\ref{alg:maopd}.
In particular, we first obtain the log total probability of the top-$k$ tokens directly from the top-$k$ log-probabilities via a log-sum-exp,
\[
\log\bigl(1-p_t^{\mathrm{tail}}\bigr)=\operatorname{logsumexp}_{v\in S_t^k}\log p_t(v),
\qquad
\log\bigl(1-q_t^{\mathrm{tail}}\bigr)=\operatorname{logsumexp}_{v\in S_t^k}\log q_t(v),
\]
which is numerically stable and never overflows.
To further rule out the boundary case $p_t^{\mathrm{tail}} \to 0$, we clamp the log total probability such that $\log(1-p_t^{\mathrm{tail}}) \le -\epsilon$, ensuring a strictly positive tail probability and a finite log-domain computation.
The tail log-probability term $\log p_t^{\mathrm{tail}}$ is then computed using the
\texttt{log1mexp} primitive, which directly evaluates $\log(1-\exp(a))$ in log space:
\[
\log p_t^{\mathrm{tail}}
= \operatorname{log1mexp}\!\bigl(\log(1-p_t^{\mathrm{tail}})\bigr).
\]
The teacher term $\log q_t^{\mathrm{tail}}$ is computed analogously. 
In this way, we avoid the numerical issues of the naive implementation.

\begin{lstlisting}[language=Python, float=tbp, caption={Pseudocode for the TA-OPD loss implementation.}, label={alg:maopd}]
def compute_taopd_loss(student_topk_log_probs, teacher_topk_log_probs, eps):
    # Input:
    #   student_topk_log_probs: student log-probs on the teacher's top-k tokens
    #   teacher_topk_log_probs: teacher top-k log-probs, log q(v)
    #   eps:                    small constant to keep the tail prob positive
    # Return:
    #   loss: per-token TA-OPD reverse KL divergence on the augmented token set

    # 1. log total probability of the top-k tokens via logsumexp
    student_log_topk_prob = logsumexp(student_topk_log_probs, dim=-1)   # log(1 - p_tail)
    teacher_log_topk_prob = logsumexp(teacher_topk_log_probs, dim=-1)   # log(1 - q_tail)

    # 2. tail-token log-prob log(p_tail) via log1mexp
    student_tail_log_prob = log1mexp(clamp_max(student_log_topk_prob, -eps))
    teacher_tail_log_prob = log1mexp(clamp_max(teacher_log_topk_prob, -eps))

    # 3. append the tail token and compute reverse KL on the augmented token set
    student_log_p = cat([student_topk_log_probs, student_tail_log_prob], dim=-1)
    teacher_log_q = cat([teacher_topk_log_probs, teacher_tail_log_prob], dim=-1)
    taopd_loss = sum(exp(student_log_p) * (student_log_p - teacher_log_q), dim=-1)
    return taopd_loss
\end{lstlisting}

\section{Extensive Study}

\subsection{Synthetic experiment}
\label{subsec:synthetic}
To provide an intuitive understanding of different OPD objectives, we construct a synthetic experiment based on a 30-armed bandit.  
Both the student and teacher policies are parameterized by vectors in \(\mathbb{R}^{1\times 30}\), and the corresponding probability distributions are obtained by applying the softmax function to these vectors.  
The teacher policy is kept fixed during training and is defined over the discrete class space \(\mathcal{V} = \{1,\dots,30\}\) as the following bimodal distribution:
\[
q(v) \propto \exp\!\left(-\frac{(v-10)^2}{2\cdot 2^2}\right) \;+\; 0.88 \cdot \exp\!\left(-\frac{(v-20)^2}{2\cdot 2^2}\right), \quad v \in \mathcal{V}.
\]
The student policy is parameterized as a categorical distribution over \(\mathcal{V}\):
\[
p_\theta(v) = \mathrm{softmax}(z_\theta)_v, \quad z_\theta \in \mathbb{R}^{30},
\]
where \(z_\theta\) is initialized from \(\mathcal{N}(0, 0.01^2)\).
The student is optimized for 20000 steps using AdamW with learning rate $1\times10^{-3}$. We compare four objectives: full-vocabulary OPD, normalized top-$k$ OPD, TA-OPD, and sample-corrected TA-OPD.
For top-$k$ OPD, $k$ is set to $8$.

\definecolor{plotpurple}{RGB}{128,0,128}

\begin{figure}
    \centering
    \includegraphics[width=1.0\linewidth]{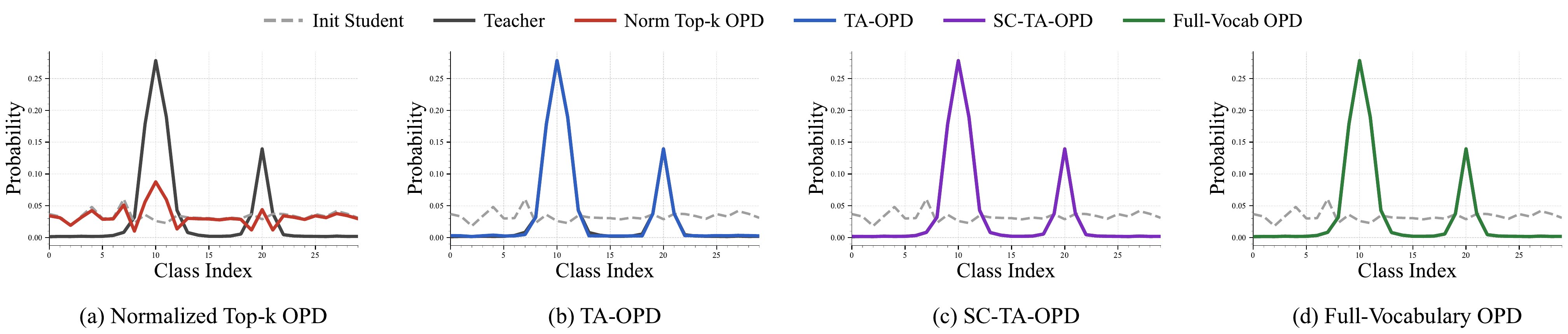}
    \caption{
Visualization of the student policies under different OPD objectives.
Normalized top-$k$ OPD (\textcolor{red}{Red}) only matches the relative shape of the teacher distribution over the top-$k$ tokens, but fails to align the tail probability.
In contrast, TA-OPD (\textcolor{blue}{Blue}) and sample-corrected TA-OPD (\textcolor{plotpurple}{Purple}) closely recover the bimodal teacher distribution.
}

    \label{fig:policy_comparison}
\end{figure}

\paragraph{TA-OPD matches both the shape and tail probability of the teacher policy.}
Figure~\ref{fig:policy_comparison} compares the final student policies trained with different OPD objectives.
TA-OPD closely recovers the bimodal teacher distribution.
While normalized top-$k$ OPD matches the relative shape over the teacher's top-$k$ tokens, it assigns a substantially larger tail probability than the teacher.
Overall, TA-OPD preserves the missing tail probability information and therefore provides a much closer approximation to the teacher policy than normalized top-$k$ OPD.

\begin{figure}
    \centering
    \includegraphics[width=1.0\linewidth]{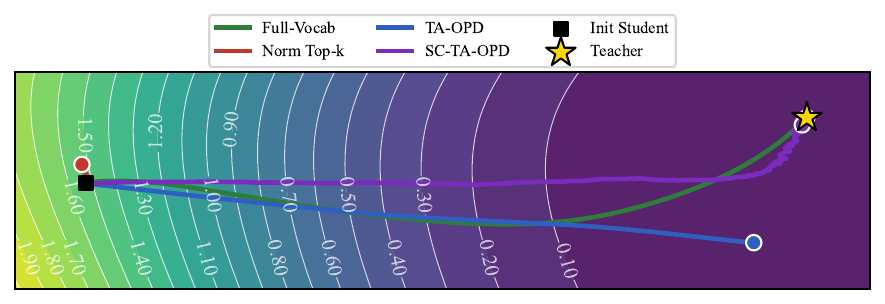}
   \caption{
\textbf{Learning dynamics of different OPD objectives.}
TA-OPD better approximates the full-vocabulary OPD.
}

    \label{fig:opd_landscape}
\end{figure}

\paragraph{TA-OPD better approximates the full-vocabulary OPD.}
Figure~\ref{fig:opd_landscape} further visualizes the optimization trajectories on the full-vocabulary reverse-KL landscape.
In particular, we collect the logit trajectories of all methods together with the teacher logits, center them in logit space, and project them onto the first two principal components.
We then evaluate the full-vocabulary reverse KL on this two-dimensional plane and overlay the trajectories of different objectives.
The results show that full-vocabulary OPD and TA-OPD move toward the same low-loss region around the teacher policy, while normalized top-$k$ OPD converges to a point that remains far from the teacher under the full-vocabulary KL. 
Overall, TA-OPD's objective provides a better top-$k$ estimate of the full-vocabulary reverse KL than the normalized objective.

\subsection{Top-$k$ Overlap Ratio of Different OPD Methods}
\label{subsec:overlap}

\citet{li2026rethinking} find that the top-$k$ overlap ratio between the student and the teacher predicts the success of OPD well.
Here we compare the two objectives under this metric.

\paragraph{Setup.}
On each student-generated prefix, the top-$k$ overlap ratio is the fraction of the
teacher's top-$k$ tokens that also fall in the student's top-$k$ tokens,
\[
\mathrm{Overlap}_t^k
= \frac{\bigl|\operatorname{TopK}(p_t,k)\cap\operatorname{TopK}(q_t,k)\bigr|}{k},
\]
which we average over all tokens in a training batch.
A higher ratio means the student and the teacher agree on which tokens are plausible,
so the top-$k$ objective supervises a token set that is meaningful to the student.
We use the three student--teacher pairs of Section~\ref{subsec:empirical_anlysis} and
report the ratio across training steps.

\begin{figure}
    \centering
    \includegraphics[width=1.0\linewidth]{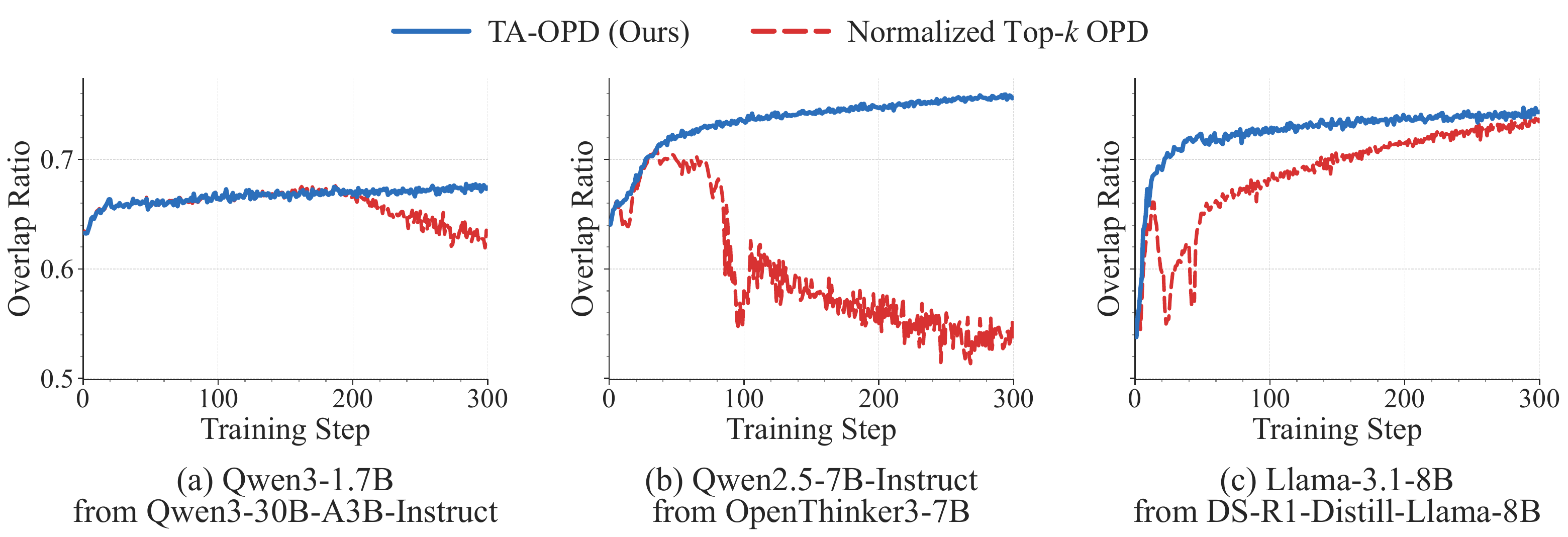}
    \caption{
\textbf{Top-$k$ overlap ratio between the student and the teacher across training steps.}
TA-OPD achieves higher overlap ratios than normalized top-$k$ OPD across different model pairs.
}
    \label{fig:overlap_ratio}
\end{figure}

\paragraph{TA-OPD attains a higher top-$k$ overlap ratio.}
Figure~\ref{fig:overlap_ratio} compares the overlap ratio of the two objectives.
TA-OPD increases the overlap ratio monotonically and keeps it stable throughout training.
In contrast, normalized top-$k$ OPD is consistently lower: for Qwen2.5-7B-Instruct the
ratio peaks early and then collapses, and for Llama-3.1-8B it recovers only slowly and
remains below TA-OPD for the entire run.
Overall, TA-OPD keeps the student and the teacher aligned on the top-$k$ tokens, which is consistent with its stronger downstream accuracy.

\subsection{Computational Overhead of TA-OPD}
\label{subsec:overhead}

TA-OPD introduces no additional computational overhead compared with normalized top-$k$ OPD.
Both methods require exactly the same teacher-side information: the log-probabilities of the teacher's top-$k$ tokens.
Given these quantities, TA-OPD only additionally computes the tail log-probabilities $\log p_t^{\mathrm{tail}}$ and $\log q_t^{\mathrm{tail}}$, which involve a single log-sum-exp over $k$ values per token.
This cost is negligible relative to the forward and backward passes of the student model.
Sample-corrected TA-OPD additionally requires the teacher's log-probability of the sampled token $\hat y_t$.
Although this appears to be an extra query, it incurs negligible additional cost in practice: when queried for the top-$k$ log-probabilities, inference engines such as vLLM~\citep{kwon2023efficient} always compute and return the sampled token's log-probability as well, even if it falls outside the top-$k$ tokens.
Therefore, the teacher-query cost of sample-corrected TA-OPD is identical to that of TA-OPD.

We empirically verify this by measuring the wall-clock cost of the sampled-token OPD, normalized top-$k$ OPD, TA-OPD, and sample-corrected TA-OPD under the same setup.
We use Qwen3-1.7B as the student and Qwen3-30B-A3B-Instruct-2507 as the teacher, with all other settings following Table~\ref{tab:opd_hyperparameter}.
As shown in Figure~\ref{fig:computational_overhead} (a), completing the full $300$ training steps takes $23$h$2$m$18$s for TA-OPD and $22$h$47$m$56$s for normalized top-$k$ OPD, i.e., a difference of $14$m$22$s, or $1.05\%$ of the total runtime.
This 14-minute gap might stem from the slightly longer average response length under TA-OPD.
Figure~\ref{fig:computational_overhead} (b) further compares the per-step throughput distributions, whose medians are around $1000$ and $1005$ tokens/s for TA-OPD and normalized top-$k$ OPD, respectively.
Overall, TA-OPD incurs negligible computational cost compared to normalized top-$k$ OPD.

\begin{figure}[t]
    \centering
    \begin{subfigure}[t]{0.48\linewidth}
        \centering
        \includegraphics[width=\linewidth]{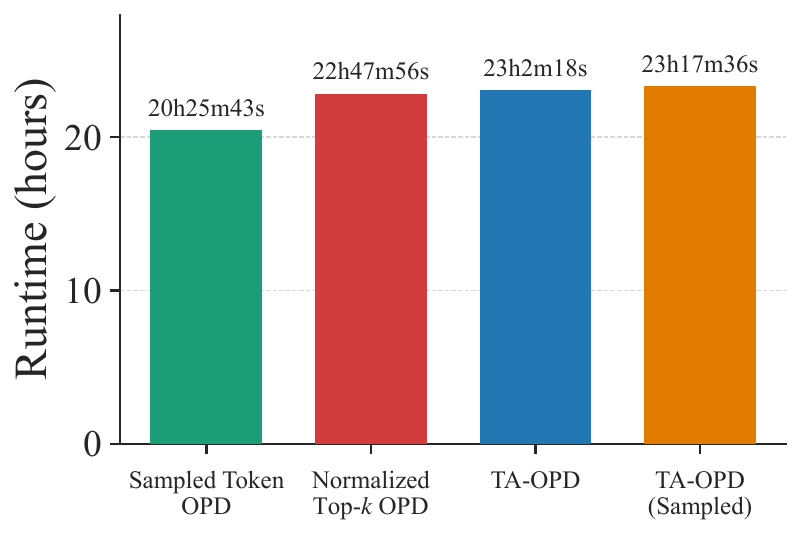}
        \caption{Total training runtime.}
        \label{fig:compare_time}
    \end{subfigure}
    \hfill
    \begin{subfigure}[t]{0.48\linewidth}
        \centering
        \includegraphics[width=\linewidth]{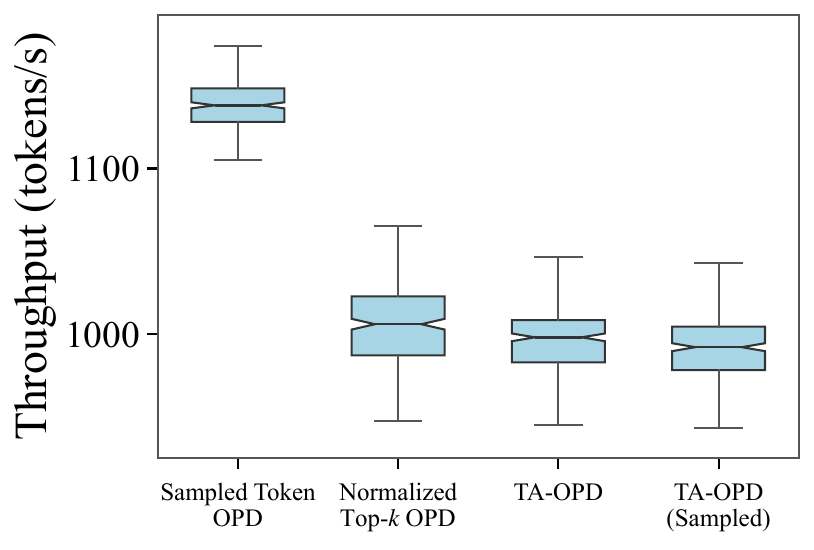}
        \caption{Throughput distribution during training.}
        \label{fig:throughput_distribution}
    \end{subfigure}

    \caption{Computational overhead comparison between sampled-token OPD, normalized Top-$k$ OPD, TA-OPD, and sample-corrected TA-OPD. 
    Left: total training runtime. Right: throughput distribution across training steps.
    The student and teacher models are Qwen3-1.7B and Qwen3-30B-A3B-Instruct-2507, respectively.
    }
    \label{fig:computational_overhead}
\end{figure}

\subsection{Ablation on the Tail Token}
\label{subsec:ablate_tail}

Our analysis attributes the tail probability increase issue to top-$k$ normalization.
A natural question is therefore whether simply removing the normalization is
already sufficient, making the tail token in TA-OPD unnecessary.
We show that it is not: without the tail token $v_{\mathrm{tail}}$, training will diverge.

\paragraph{Setup.}
We ablate the tail token by comparing TA-OPD with unnormalized top-$k$ OPD, which removes the tail term in TA-OPD's loss function. 
Specifically, its loss function is given by:
\[
\ell_t^{\mathrm{unnorm}} = \sum_{v \in S_t^k} p_t(v)\log\frac{p_t(v)}{q_t(v)}
= \ell_t^{\mathrm{TA}} - p_t^{\mathrm{tail}}\log\frac{p_t^{\mathrm{tail}}}{q_t^{\mathrm{tail}}}.
\]

The student and teacher are Qwen2.5-7B-Instruct and OpenThinker3-7B, with all
other settings following Table~\ref{tab:opd_hyperparameter}.

\begin{figure}
    \centering
    \includegraphics[width=1.0\linewidth]{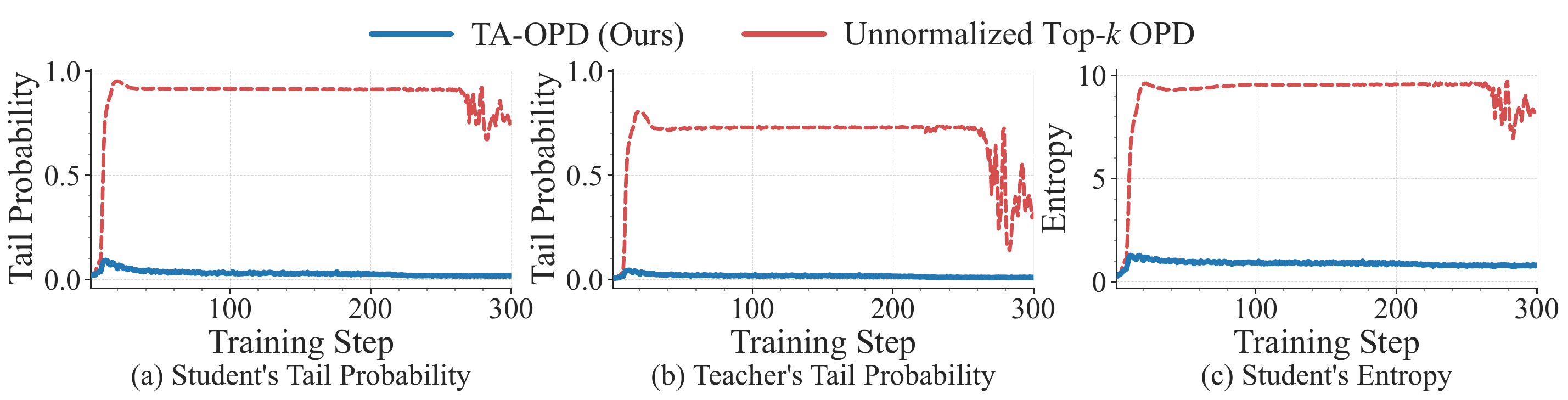}
    \caption{
\textbf{The student's tail probability, the teacher's tail probability, and the student's
token-level entropy across training steps.}
Without the tail token, the student's tail probability rises above $0.9$ and training collapses.
}
    \label{fig:ablate_tail_token}
\end{figure}

\begin{table}[t]
    \centering
    \caption{
    \textbf{Avg@8 of unnormalized top-$k$ OPD and TA-OPD.}
    The \textbf{Avg.} column represents the macro-average across all six benchmarks.
    Best results are shown in \textbf{bold}.
    }
    \label{tab:compare_unnorm}

    \setlength{\tabcolsep}{5pt}
    \renewcommand{\arraystretch}{1.15}
    \resizebox{\textwidth}{!}{
        \begin{tabular}{l|cccccc|c}
        \toprule
        \textbf{Methods}
        & \textbf{MATH500}
        & \textbf{Minerva}
        & \textbf{Olympiad}
        & \textbf{AMC}
        & \textbf{AIME 24}
        & \textbf{AIME 25}
        & \textbf{Avg.} \\
        \midrule

        \multicolumn{8}{c}{
            Student: \textbf{\textit{Qwen2.5-7B-Instruct}}
            \quad
            Teacher: \textbf{\textit{OpenThinker3-7B}}
        } \\
        \midrule

        UnNorm.\ top-$k$
        & 3.35
        & 1.70
        & 1.42
        & 3.00
        & 0.42
        & 0.42
        & 1.72 \\

     \rowcolor{blue!10}
    \textbf{TA-OPD (Ours)}
    & \textbf{77.88}
    & \textbf{32.58}
    & \textbf{42.41}
    & \textbf{45.03}
    & \textbf{16.67}
    & \textbf{17.50}
    & \textbf{38.68} \\
    \bottomrule
        \end{tabular}
    }
\end{table}

\paragraph{Removing the tail token collapses training.}
Figure~\ref{fig:ablate_tail_token} compares the training dynamics of the two objectives.
Under unnormalized top-$k$ OPD, the student's tail probability rises above $0.9$ within
50 steps and its entropy grows monotonically, while the teacher's tail probability
increases in tandem.
The collapse is reflected in downstream accuracy (Table~\ref{tab:compare_unnorm}):
its Avg@8 on MATH500 is only 3.35\%, whereas TA-OPD reaches 77.88\%.

\paragraph{The unnormalized objective is not a divergence.}
The failure is not an optimization artifact but a property of the objective itself:
the objective is not minimized at $p_t=q_t$.
Its minimizer admits the closed form
\[
p_t^\star(v) = q_t(v)/e \quad \forall v\in S_t^k,
\qquad\text{hence}\qquad
p_t^{\mathrm{tail},\star} = 1-\frac{1-q_t^{\mathrm{tail}}}{e} \;\ge\; 1-e^{-1}\approx 0.63 .
\]
That is, the objective is minimized by deflating every top-$k$ probability by a
factor of $e$ and assigning the removed mass to the tail, which drives the loss negative.
The optimum therefore assigns at least $63\%$ of the probability mass outside the
teacher's top-$k$ tokens, regardless of the teacher distribution.
Adding the tail token contributes exactly the missing term
$p_t^{\mathrm{tail}}\log(p_t^{\mathrm{tail}}/q_t^{\mathrm{tail}})$, which turns the
objective into a reverse KL divergence on $S_t^{+}$: it is non-negative and uniquely
minimized when the student matches the teacher on both the top-$k$ tokens and the tail.

Overall, introducing the tail token is necessary for TA-OPD.

\subsection{Discussion of EMA-PG}
\label{subsec:ema_pg}

Exponential Moving Average-Policy Gradient (EMA-PG)~\citep{zhang2026ema} proposes a top-$k$ KL estimator that is related but distinct from our sample-corrected TA-OPD.
Its per-token loss is:
\begin{equation}
\ell_t^{\mathrm{EMA}}
=
\underbrace{
\sum_{v\in S_t^{k}}
p_t(v)\operatorname{sg}\!\left(\log\frac{p_t(v)}{q_t(v)}\right)
}_{\textstyle =\;\ell_t^{\mathrm{unnorm}}\ \text{\normalsize in value}}
+
\mathbf{1}[\hat y_t\notin S_t^k]\,
\frac{p_t(\hat y_t)}{\operatorname{sg}(p_t(\hat y_t))}
\operatorname{sg}\!\left(\log\frac{p_t(\hat y_t)}{q_t(\hat y_t)}\right).
\label{eq:ema}
\end{equation}
In contrast, sample-corrected TA-OPD's per-token loss is given by:
\begin{equation*}
\ell_t^{\mathrm{SC\text{-}TA}}
=
\underbrace{
\sum_{v\in S_t^{+}}
p_t(v)\operatorname{sg}\!\left(\log\frac{p_t(v)}{q_t(v)}\right)
}_{\textstyle =\;\ell_t^{\mathrm{TA}}\ \text{\normalsize in value}}
+
\mathbf{1}[\hat y_t\notin S_t^k]\,
\frac{p_t(\hat y_t)}{\operatorname{sg}(p_t(\hat y_t))}
\operatorname{sg}\!\left(\log\frac{p_t(\hat y_t)/p_t^{\mathrm{tail}}}{q_t(\hat y_t)/q_t^{\mathrm{tail}}}\right).
\end{equation*}
Both estimators are unbiased in value and in gradient to full-vocabulary reverse KL divergence, but they differ in where the truncation is placed and in what the sampled token is used to estimate.
EMA-PG estimates the entire tail contribution $\sum_{v\notin S_t^k}p_t(v)\log\frac{p_t(v)}{q_t(v)}$ with the sampled token.
Instead, our estimator starts from $\ell_t^{\mathrm{TA}}$, which already accounts for the
tail probability through the tail token, and uses the sampled token only to estimate the
residual $p_t^{\mathrm{tail}}\DKL(\tilde p_t\|\tilde q_t)$ identified in
Proposition~\ref{prop:bound}.
In other words, our estimator is obtained by debiasing TA-OPD, whereas EMA-PG debiases the unnormalized top-$k$ objective.

\paragraph{Behavior when the sampled token falls inside the top-$k$ tokens.}
When $\hat y_t\in S_t^k$, which occurs with probability $1-p_t^{\mathrm{tail}}$ and thus
covers most tokens, the correction term vanishes for both losses.
In this case, $\ell_t^{\mathrm{EMA}}$ reduces to the unnormalized top-$k$ loss
$\ell_t^{\mathrm{unnorm}}$, while $\ell_t^{\mathrm{SC\text{-}TA}}$ reduces to
$\ell_t^{\mathrm{TA}}$.
As shown in Appendix~\ref{subsec:ablate_tail}, $\ell_t^{\mathrm{TA}}$ is superior to $\ell_t^{\mathrm{unnorm}}$.
Consequently, although EMA-PG is unbiased in expectation, on most tokens its realized
objective is the unnormalized top-$k$ loss function that we have shown to collapse training, and it relies on infrequent
sampled-token corrections to compensate; our estimator instead falls back to a well-defined divergence on $S_t^{+}$ that is uniquely minimized at $p_t=q_t$.

\paragraph{Variance analysis.}
Since both estimators are unbiased, a natural criterion for comparing them is the variance
of the estimate.
A direct computation gives
\begin{equation}
\operatorname{Var}\bigl(\ell_t^{\mathrm{EMA}}\bigr)-\operatorname{Var}\bigl(\ell_t^{\mathrm{SC\text{-}TA}}\bigr)
=p_t^{\mathrm{tail}}\bigl(1-p_t^{\mathrm{tail}}\bigr)\,
\log\frac{p_t^{\mathrm{tail}}}{q_t^{\mathrm{tail}}}
\left(\log\frac{p_t^{\mathrm{tail}}}{q_t^{\mathrm{tail}}}
+2\,\DKL(\tilde p_t\|\tilde q_t)\right).
\label{eq:var_gap}
\end{equation}
Since $\DKL(\tilde p_t\|\tilde q_t)\ge 0$, the gap is strictly positive whenever
$p_t^{\mathrm{tail}}>q_t^{\mathrm{tail}}$, i.e.\ whenever the student's tail probability
exceeds the teacher's, which is likely to happen when using the teacher's top-$k$ support.
Since the support $S_t^k$ consists of the teacher's top-$k$ tokens, the teacher's tail probability is typically smaller than the student's.
Figures~\ref{fig:mass_entropy_comparison} confirms this
empirically, where $q_t^{\mathrm{tail}}$ is below $p_t^{\mathrm{tail}}$ throughout
training across all model pairs.
Overall, our loss objective tends to have lower variance than EMA-PG.

\section{Detailed Results}

\subsection{Training Dynamics of Additional Model Pairs}
Figure~\ref{fig:mass_entropy_comparison_multi} presents additional comparisons between TA-OPD and normalized top-$k$ OPD in terms of tail probability and token-level entropy during training.
Results are shown on three student--teacher pairs: Qwen3-1.7B paired with Qwen3-30B-A3B-Instruct-2507, Llama-3.1-8B paired with DeepSeek-R1-Distill-Llama-8B, Qwen3-1.7B-Base paired with Qwen3-8B.

\begin{figure}
    \centering


     \begin{subfigure}[t]{1.0\linewidth}
        \centering
        \includegraphics[width=\linewidth]{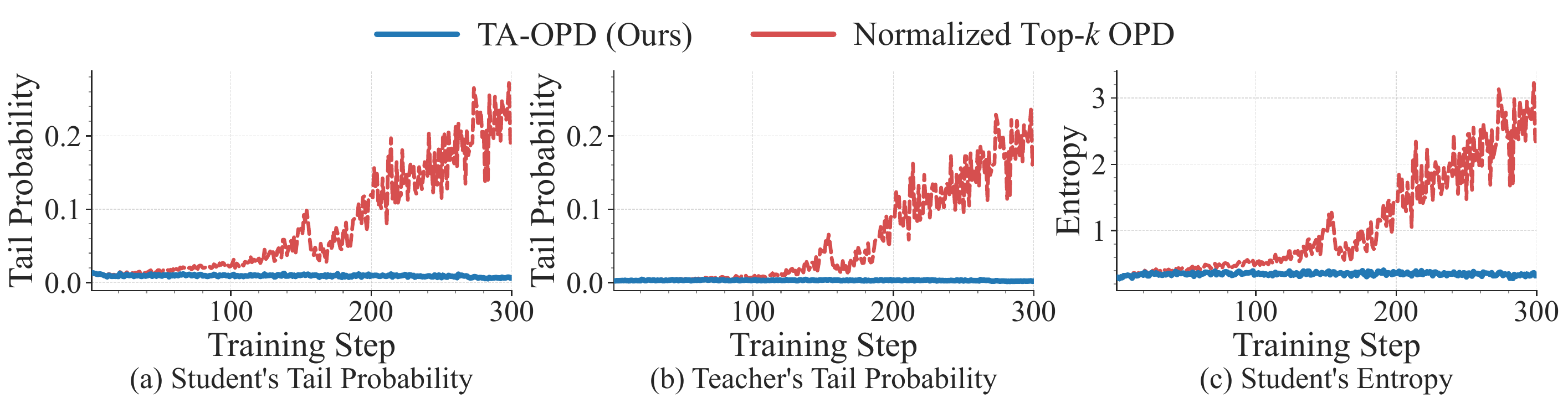}
        \caption*{\textbf{Qwen3-1.7B (Student)} $\rightarrow$ \textbf{Qwen3-30B-A3B-Instruct-2507 (Teacher)}}
    \end{subfigure}
    
    \begin{subfigure}[t]{1.0\linewidth}
        \centering
        \includegraphics[width=\linewidth]{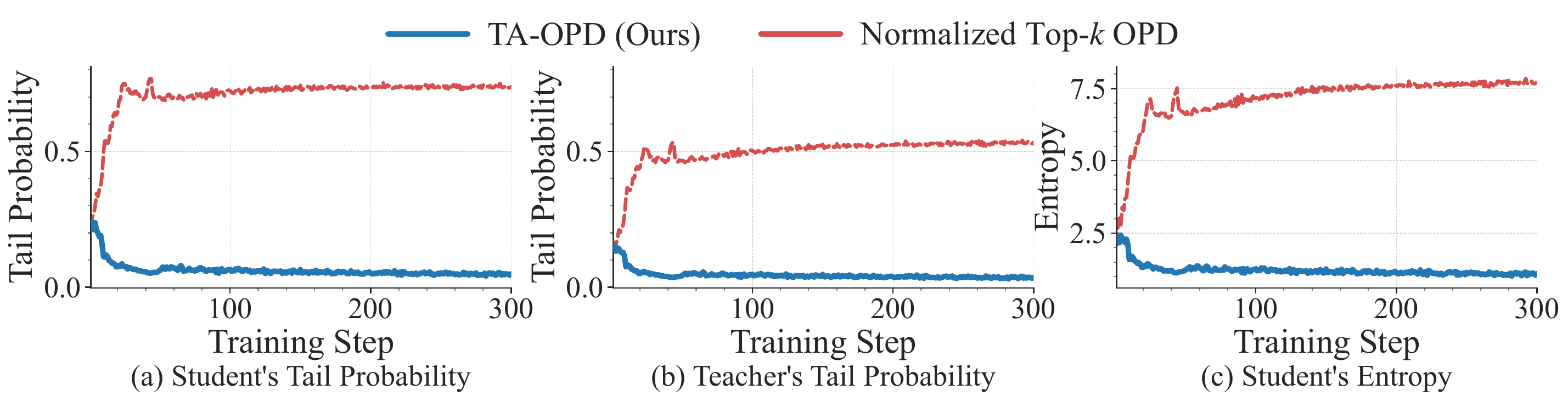}
        \caption*{\textbf{Llama-3.1-8B (Student)} $\rightarrow$ \textbf{DeepSeek-R1-Distill-Llama-8B (Teacher)}}
    \end{subfigure}

    \begin{subfigure}[t]{1.0\linewidth}
        \centering
        \includegraphics[width=\linewidth]{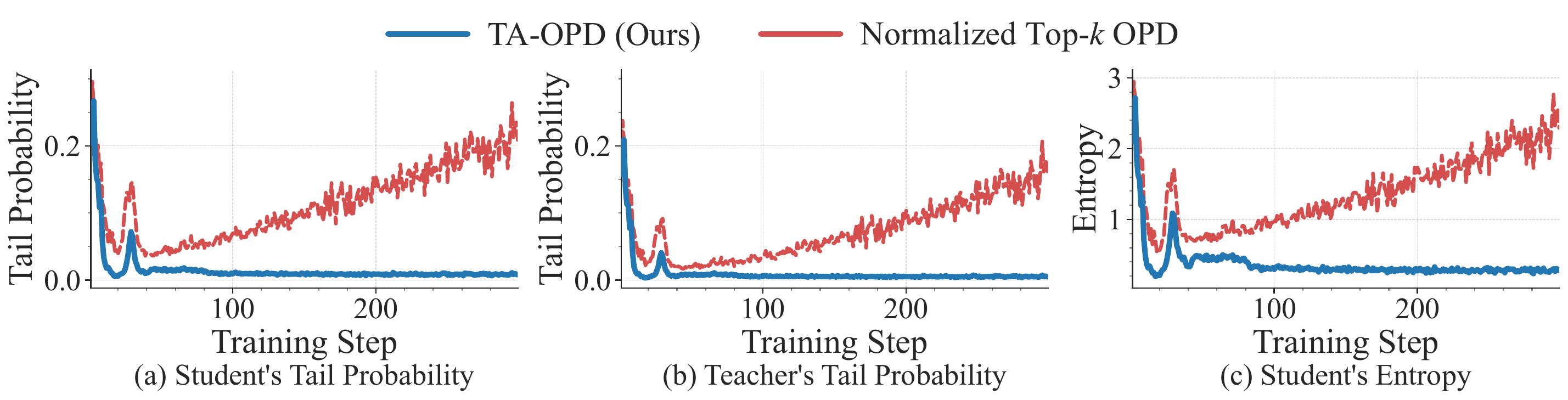}
        \caption*{\textbf{Qwen3-1.7B-Base (Student)} $\rightarrow$ \textbf{Qwen3-8B (Teacher)}}
    \end{subfigure}

    \caption{
\textbf{Tail probability and token-level entropy over training for TA-OPD and Normalized top-$k$ OPD.}
Each row corresponds to a different student-teacher pair.
}
    \label{fig:mass_entropy_comparison_multi}
\end{figure}

\begin{figure}
    \centering
        \begin{subfigure}[t]{1.0\linewidth}
        \centering
        \includegraphics[width=\linewidth]{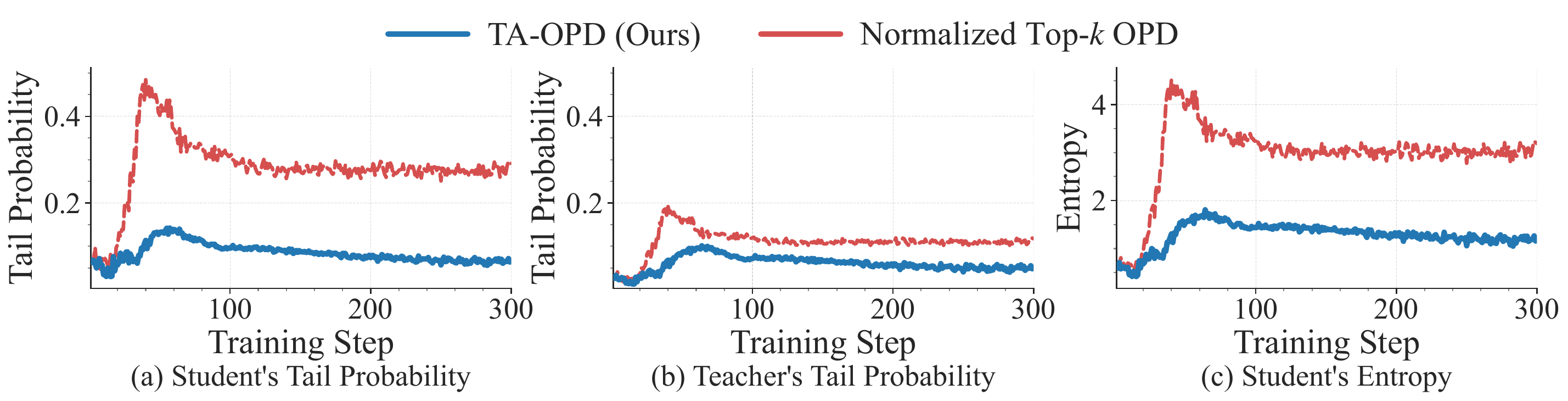}
        \caption*{\textbf{Qwen2.5-Math-1.5B (Student)} $\rightarrow$ \textbf{JustRL-DeepSeek-1.5B (Teacher)}}
    \end{subfigure}
    
     \begin{subfigure}[t]{1.0\linewidth}
        \centering
        \includegraphics[width=\linewidth]{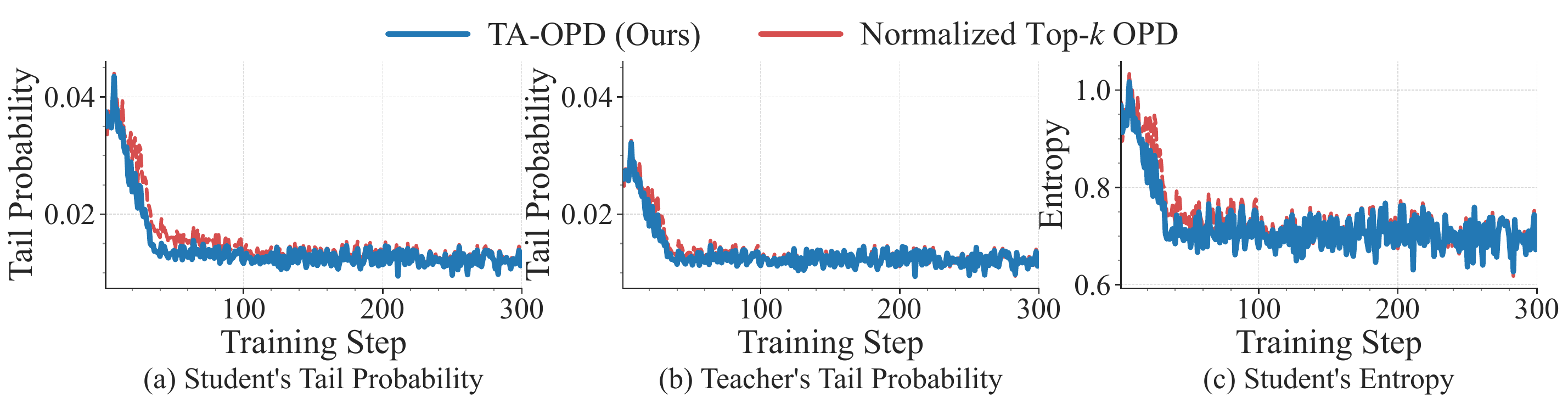}
        \caption*{\textbf{DeepSeek-R1-Distill-Qwen-1.5B (Student)} $\rightarrow$ \textbf{JustRL-DeepSeek-1.5B (Teacher)}}
    \end{subfigure}
    \caption{\textbf{Tail probability and token-level entropy over training for TA-OPD and Normalized top-$k$ OPD under different student-teacher capability gaps.}
The tail probability increase only occurs for Qwen2.5-Math-1.5B but not for DeepSeek-R1-Distill-Qwen-1.5B.}
    \label{fig:capability_gap_tail_prob}
\end{figure}

\end{document}